\documentclass[11pt,en,cite=authoryear]{note}

\title{Time Series Generative Learning with Application to Brain Imaging Analysis}
\author{Zhenghao Li$^*$ \\	$*$ Contributes equally \\ \and Sanyou Wu$^*$ \\ $\#$ Correspondence to: lfeng@hku.hk  \\ \and Long Feng$^\#$ 
}
\institute{Department of Statistics \& Actuarial Science, The University of Hong Kong}
\date{}

\usepackage{appendix}
\usepackage{array}
\usepackage{subfigure}
\usepackage{graphicx}
\usepackage{amssymb,amsmath,amsthm}

\usepackage{xcolor}
\usepackage{graphicx}
\graphicspath{
	{note/figure/}{note/image}{./image/}
}
\let\savehash\hash
\let\hash\relax
\usepackage{mathabx}
\let\hash\savehash
\definecolor{pinegreen}{rgb}{0.0, 0.47, 0.44}
\usepackage{algorithm}
\usepackage{algorithmic}
\usepackage{amsmath}
\usepackage{amsfonts}
\usepackage{overpic}

\definecolor{pinegreen}{rgb}{0.0, 0.47, 0.44}
\definecolor{MyYellow}{HTML}{FF9C3B}
\definecolor{MyRed}{HTML}{EA6060}
\definecolor{MyGreen}{HTML}{8ACC82}
\definecolor{MyPurple}{HTML}{8286CC}

\newtheorem{thm}{\textbf{Theorem}}
\newtheorem{prop}{\textbf{Proposition}}
\newtheorem{col}{Corollary}
\newtheorem{lem}{Lemma}

\newtheorem{asmp}{Asumption}
\newtheorem{exm}{Example}

\def\bbR{{\mathbb{R}}}

\def\argmin{\operatorname{argmin} \displaylimits}

\def\E{\mathbb{E}}

\def\T{\top}

\def\Deltabreve{\breve{\Delta}}
\def\Deltabar{\widebar{\Delta}}
\def\Deltatilde{\widetilde{\Delta}}

\def\hX{\widehat{X}}

\def\bbL{{\mathbb{L}}}
\def\bbE{{\mathbb{E}}}

\newcommand{\bel}{\begin{eqnarray}\label}
\newcommand{\eel}{\end{eqnarray}}
\newcommand{\bes}{\begin{eqnarray*}}
\newcommand{\ees}{\end{eqnarray*}}
\newcommand{\bei}{\begin{itemize}}
\newcommand{\beiftnt}{\begin{itemize}\footnotesize}
\newcommand{\eei}{\end{itemize}}

\def\benu{\begin{enumerate}}
\def\eenu{\end{enumerate}}

\def\argmin{\mathop{\rm arg\, min}}

\def\E{{\mathbb{E}}}

\def\complex{\mathop{{\rm I}\kern-.58em\hbox{\rm C}}\nolimits}

\def\diag{\hbox{\rm diag}}

\def\mathbold{\boldsymbol} 

\def\calB{{\cal B}}

\def\calD{{\cal D}}

\def\ghat{\widehat{g}}

\def\calG{{\cal G}}\def\Ghat{\widehat{G}}

\def\hhat{\widehat{h}}

\def\calH{{\cal H}}\def\Hhat{\widehat{H}}

\def\bI{\mathbold{I}}

\def\calK{{\cal K}}

\def\calL{{\cal L}}

\def\calN{{\cal N}}

\def\bs{\mathbold{s}}

\def\calW{{\cal W}}

\def\Xhat{\widehat{X}}\def\Xtil{{\widetilde X}}

\def\Yhat{\widehat{Y}}

\def\Deltabar{{\overline \Delta}}

\begin{document}
\maketitle

\begin{abstract}
This paper focuses on the analysis of sequential image data, particularly brain imaging data such as MRI, fMRI, CT, with the motivation of understanding the brain aging process and neurodegenerative diseases. To achieve this goal, we investigate image generation in a time series context. Specifically, we formulate a min-max problem derived from the $f$-divergence between neighboring pairs to learn a time series generator in a nonparametric manner. The generator enables us to generate future images by transforming prior lag-k observations and a random vector from a reference distribution. With a deep neural network learned generator, we prove that the joint distribution of the generated sequence converges to the latent truth under a Markov and a conditional invariance condition. Furthermore, we extend our generation mechanism to a panel data scenario to accommodate multiple samples. The effectiveness of our mechanism is evaluated by generating real brain MRI sequences from the Alzheimer's Disease Neuroimaging Initiative.  These generated image sequences can be used as data augmentation to enhance the performance of further downstream tasks, such as Alzheimer's disease detection.
\end{abstract}

\section{Introduction}
\noindent Time series data is not limited to numbers; it can also include images, texts, and other forms of information. This paper specifically focuses on the analysis of image time series, driven by the goal of understanding brain aging process with brain imaging data, such as magnetic resonance imaging (MRI), computer tomography (CT), etc.
Understanding the brain aging process holds immense value for a variety of applications. For instance, insights into the brain aging process can reveal structural changes within the brain \citep{tofts2005quantitative} and aid in the early detection of degenerative diseases such as Alzheimer's Disease \citep{jack2004comparison}. Consequently, the characterization and forecasting of brain image time series data can play a crucial role in the fight against age-related neurodegenerative diseases \citep{cole2018brain, huizinga2018spatio}.

Time series analysis, a classic topic in statistics, has been extensively studied, with numerous notable works such as autoregression (AR), autoregressive moving-average 
(ARMA, \citealt{brockwell1991time}), autoregressive conditional heteroskedasticity (ARCH, \citealt{engle1982autoregressive, bollerslev1986generalized}), among others. The classic scalar time series model has been expanded to accommodate vectors \citep{stock2001vector}, matrices \citep{chen2021autoregressive,han2023simultaneous,chang2023modelling}, and even tensors \citep{chen2022factor} in various contexts.
Furthermore, numerous time series analysis methods have been adapted to high-dimensional models, employing regularization techniques, such as \citet{basu2015regularized,guo2016high}. Beyond linear methods, nonlinear time series models have also been explored in the literature, e.g., \citet{fan2003nonlinear, tsay2018nonlinear}, etc. We refer to \cite{tsay2013multivariate} for a comprehensive review of time series models. In recent years, with the advent of deep learning, various deep neural network architectures have been applied to sequential data modeling, including recurrent neural network (RNN)-based methods \citep{pascanu2013construct, lai2018modeling} and attention-based methods \citep{li2019enhancing, zhou2021informer}. These deep learning models have demonstrated remarkable success in addressing nonlinear dependencies, resulting in promising forecasting performances.  
However, in spite of these advances, image data typically contains rich spatial and structural information along with substantially higher dimensions. Time series methods designed for numerical values are not easily applicable to image data analysis.

In recent years, imaging data analysis has attracted significant attention in both statistics and machine learning community. Broadly speaking, existing work in the statistical literature can be divided into two categories. The first category typically involves converting images into vectors and analyzing the resulting vectors with high-dimensional statistical methods, such as Lasso \citep{TibshiraniR96}  or other penalization techniques, Bayesian approaches \citep{kang2018scalar,feng2019bayesian}, etc.
Notably, this type of approach may generate ultra high-dimensional vectors and face heavy computational burdens when dealing with high resolution images.
The second category of approaches treats image data as tensor inputs and employs general tensor decomposition techniques, such as Canonical polyadic decomposition \citep{zhou2013tensor} or Tucker decomposition \citep{li2018tucker,luo2023low}.
Recently, Kronecker product based decomposition has been applied in image analysis and its connections to Convolutional Neural Network (CNN,  \citealt{lecun1998gradient}) has been explored in the literature \citep{wu2023sparse,feng2023deep}. Indeed,
Deep neural networks (DNN), particularly CNNs, have arguably emerged as the  most prevalent approach for image data analysis across various contexts.

Owing to the success of DNNs, generative models have become an immense active area, achieving numerous accomplishments in image analysis and computer vision. In particular, variational autoencoder \citep{kingma2013auto}, generative adversiral networks (GAN, \citealt{goodfellow2014generative}) and diffusion models \citep{ho2020denoising}  have gained substantial attention and have been applied to various imaging tasks. Moreover, statisticians have also made tremendous contributions to the advancement of generative models. For example, \cite{zhou2023deep} introduced a deep generative approach for conditional sampling, which matches appropriate joint distributions using Kullback-Liebler divergence. \cite{zhou2023testing} introduced a non-parametric test for the Markov property in time series using generative models.
\citet{chen2022inferential} proposed an inferential Wasserstein GAN (iWGAN) model that fuses autoencoders and GANs, and established the  generalization error bounds for the iWGAN. 
Although generative models have achieved remarkable success, their investigation in the context of image time series analysis has been much less explored, to the best of our knowledge.   \citet{ravi2019degenerative} introduced a deep learning approach to generate  neuroimages and simulate disease progression, while \citet{yoon2023sadm} employed a conditional diffusion model for generating sequential medical images. 
In fact, handling medical imaging data can present additional challenges, such as limited sample sizes, higher image resolutions, and subtle differences between images. Consequently, conventional generative approaches might not be sufficient for generating medical images.

Motivated by brain imaging analysis, this paper considers image generation problem in a time series context. Suppose that a time series of images $\{X_t\in\mathbb{R}^p,  t=1,\ldots, T\}$ of dimension $p$ is observed from $0$ to $T$ and we are interested in generating the next $S$ points, i.e., $\{X_t, t=T+1,\ldots, T+S\}$. Our goal is to generate a sequence $(\Xhat_{T+1},\ldots, \Xhat_{T+S})$ which follows the same joint distribution as $(X_{T+1},\ldots, X_{T+S})$, i.e., 
\bel{intro1}
(\Xhat_{T+1},\cdots, \Xhat_{T+S})  \overset{\text{d}}{=} (X_{T+1},\cdots,X_{T+S}).
\eel
In this paper, we show that under a Markov and a conditional invariance condition, such generation is possible by letting
\bel{intro2}
\Xhat_T = X_{T}, \ \Xhat_{T+s}  =  g(\eta_{T+s -1}, \Xhat_{T+s-1}), \quad 1 \leq s \leq S,
\eel
where $g$ is certain unknown measurable function, i.e., generator, that need to be estimated, $\eta_t$ is a $m$-dimensional random vector that is drawn i.i.d. from a reference distribution such as Gaussian. If (\ref{intro1}) could be achieved, not only would any subset of the generated sequence $(\Xhat_{T+1},\ldots, \Xhat_{T+S})$ follows the same distribution as that of $(X_{T+1},\ldots, X_{T+S})$, but the dependencies between $(\Xhat_{T+1},\ldots, \Xhat_{T+S})$ and $X_T$ also maintain consistency with the existing relationships between $(X_{T+1},\ldots, X_{T+S})$ and $X_T$. We refer to the generation (\ref{intro2}) as ``iteration generation'' as $\Xhat_{T+S}$ is generated iteratively. Additionally, we consider a ``s-step'' generation that allows us to generate  $\Xhat_{T+s}$ directly from $X_T$:
\bel{intro3}
\Xtil_{T+s} = G(\eta_{T}, X_T, s), \quad 1\leq s \leq S,
\eel
where $G$ is the target function that to be learned.

To guarantee the generation (\ref{intro2}) achieves distribution matching (\ref{intro1}), in this paper, we first establish the existence of the generator $g$. Given its existence, we formulate a min-max problem that derived from the $f$-divergence between the pairs $(X_{t}, X_{t+1})$  and $(X_{t}, g(\eta_t, X_{t}))$ to estimate the generator $g(\cdot)$. With the learned $\ghat(\cdot)$, we prove that under a mild distribution condition on $X_t$, not only does the pairwise distribution of $(X_T, \Xhat_{T+s})$ converge in expectation to $(X_T, X_{T+s})$ for any $1 \leq s \leq S$, but more important, the joint distribution of the generated sequence $(X_{T},\Xhat_{T+1},\cdots, \Xhat_{T+S})$ converges to that of $(X_T,X_{T+1},\cdots,X_{T+S})$.
The lag-1 generation in (\ref{intro2}) could be further extended to a lag-k time series context with similar distribution matching guarantees.
Furthermore, we generalize our framework to a panel data scenario to accommodate multiple samples.  Finally, 
The effectiveness of our generation mechanism is evaluated by generating real brain MRI sequences from the Alzheimer's Disease Neuroimaging Initiative (ADNI).

The rest of this paper is organized as follows. In Section \ref{sec:2}, we establish the existence of the generator and formulate the estimation of iterative and s-step generation. In Section \ref{sec:3}, we prove the weak convergence of joint and pairwise distribution for the generation.  We extend our theoretical analysis to a lag-k  time series context in Section  \ref{sec:4}. Section \ref{sec:5} generalize the framework to a panel data scenario with multiple samples. We conduct simulation studies in Section \ref{sec:simulation} and generate real brain MRI sequences in ADNI in Section \ref{sec:real}. The detailed proof and implementation of our approach are deferred to the supplementary material. 

\section{Existence and estimation}\label{sec:2}
\noindent We consider a time series $\left\{X_t \in\mathbb{R}^p, t=1,2,\ldots \right\}$ 
that satisfies the following two assumptions
\bes
&&\text{(i) Markov}: \ \ X_t\vert X_{t-1},\ldots,X_0 \ \overset{\text{d}}{=} \ X_t\vert X_{t-1},
\cr && \text{(ii) Conditional invariance}: \ \ p_{X_t\vert X_{t-1}}(x\vert y ) = p_{X_1\vert X_0}(x\vert y), \ t=1,2,\ldots.
\ees
Markov and conditional invariance property are commonly imposed in the time series analysis. Here we restrict our attention to a lag-1 time series , further generalization to lag-k scenarios will be considered in Section \ref{sec:4}. 
Notably, \citet{zhou2023testing} proposed a nonparametric test for the Markov property using generative learning. The proposed test could even allow us to infer the order of a Markov model. We omit the details here and refer interested readers to their paper for further details.


Suppose that the time series is observed from $0$ to $T$ and we are interested in generating the next $S$ points, i.e., $\{X_t, t=T+1,\ldots, T+S\}$. 
In particular, we aim to generate a sequence $(\Xhat_{T+1},\ldots, \Xhat_{T+S})$ that follows the same joint distribution as $(X_{T+1},\ldots, X_{T+S})$. More aggressively, we aim to achieve
\bel{joint1}
(X_T,\Xhat_{T+1},\cdots, \Xhat_{T+S}) \sim (X_T,X_{T+1},\cdots,X_{T+S}).
\eel
The target (\ref{joint1}) is aggressive.  If it holds true, then not only does any subset of the generated sequence $(\Xhat_{T+1},\ldots, \Xhat_{T+S})$ follows the same distribution as that of $(X_{T+1},\ldots, X_{T+S})$, but the dependencies between $(\Xhat_{T+1},\ldots, \Xhat_{T+S})$ and $X_T$ also remain consistent with those existing between $(X_{T+1},\ldots, X_{T+S})$ and $X_T$.

In this paper, we show that such a generation is possible by letting 
\bel{gen}
\Xhat_T = X_T, \ \ 
\Xhat_{T+s} = g(\eta_{T+s-1},\Xhat_{T+s-1}), \ \ 1\le s \le S, 
\eel
where $g$ is certain unknown measurable function, $\left\{\eta_t\right\}\overset{\text{i.i.d}}{\sim}N(0,\bI_m)$ is a sequence of $m$-dimensional Gaussian vector that is independent of $X_t$. 

The existence of target function $g$ is based on the following proposition. 

\begin{thm}\label{thm0}
	Let $X_0,X_1,\cdots,$ be a sequence of random variables which satisfy:
	\bes
	&&\text{(i) Markov}: \ \ X_t\vert X_{t-1},\cdots,X_0 \overset{\text{d}}{=} X_t\vert X_{t-1}
	\cr && \text{(ii) Conditional invariance}: \ \ p_{X_t\vert X_{t-1}}(x\vert y ) = p_{X_1\vert X_0}(x\vert y)
	\ees
	for any $t \ge 1$. Suppose $\Xhat_0^0 \overset{\text{d}}{=} X_0$ and  $\eta_0,\eta_1,\cdots$ be a sequence of independent $m$-dimensional Gaussian vector. Then there exist a 
	a measurable function $g$ such that the sequence 
	\bel{prop1-1}
	\left\{\Xhat_t^0: \ \Xhat_t^0 = g(\eta_{t-1},\Xhat^0_{t-1}), \ t=1,2,\ldots \right\}	
	\eel
	satisfies that for any $s\ge 1$, 
	\bel{prop1-2}
	(\Xhat_0^0,\Xhat_1^0,\cdots,\Xhat_s^0)	\overset{\text{d}}{=}  (X_0,X_1,\cdots,X_s) .
	\eel
\end{thm}
\noindent Theorem \ref{thm0} could be proved using the following  noice-outsourcing lemma.
\begin{lem}\label{lm1}
	(Theorem 5.10 in \cite{Kallenberg2021FoundationsOM})
	
	Let $(X,Y)$ be a random vector. Suppose random variables $\Xhat\overset{\text{d}}{=}X$ and $\eta \sim N(0,1)$ which is independent of $\Xhat$. Then there exist a random variable $\Yhat$ and a measurable function $g$ such that $\Yhat= g(\eta, \Xhat)$ and $(X,Y)\overset{\text{d}}{=}(\Xhat,\Yhat)$.
\end{lem}

We refer to the generation mechanism (\ref{prop1-1}) as ``iterative generation" since it is produced iteratively, one step at a time.
Theorem \ref{thm0} proves the existence of such iterative generation process. 
Besides iterative generation, an alternative approach is to directly generate the outcomes after $s$ steps. Specifically, we consider the $s$-step generation of the following form
\bel{gen2}
\Xtil_{T+s} = G(\eta_{T}, X_T, s), \ \ 1\le s \le S, 
\eel
with $G$ being the target function.



\begin{thm}\label{thm1}
	Let $X_0,X_1,\cdots,$ be a sequence of random variables which satisfy:
	\bes
	&&\text{(i) Markov}: \ \ X_t\vert X_{t-1},\cdots,X_0 \overset{\text{d}}{=} X_t\vert X_{t-1},
	\cr && \text{(ii) Conditional invariance}: \ \ p_{X_t\vert X_{t-1}}(x\vert y ) = p_{X_1\vert X_0}(x\vert y)
	\ees
	for any $t \ge 1$. 
	Let $\eta_0,\eta_1,\cdots$ be a sequence of independent $m$-dimensional Gaussian vector. Further let $g(\cdot)$ be the target function in Theorem \ref{thm0}.
	Then there exist a measurable function $G$ such that  the sequence 
	\bel{thm1-00}
	\left\{\Xtil_t^0: \ \Xtil_t^0 =G(\eta_{T}, X_0, s), \ t=1,2,\ldots \right\}	
	\eel
	satisfies
	\bel{thm1-1}
	\Xtil^0_t\overset{\text{d}}{=}X_t
	\eel
	and
	\bel{thm1-0}
	G(\eta, X ,1)= g(\eta, X), \ \ \ \forall \eta\in \mathbb{R}^m, \ \  X\in\mathbb{R}^p.
	\eel
\end{thm}


\begin{remark}\label{rmk1}
	Unlike Theorem \ref{thm0}, the sequence $\big\{\Xtil^0_t, \ t=1,2\ldots \big\}$ does not necessarily achieve the target property (\ref{prop1-2}) on the joint distribution. Instead, Theorem \ref{thm1} could only guarantee the distributional match on the marginals as in (\ref{thm1-1}). The major difference is that in the s-step generation, the  conditional distribution
	$\Xtil^0_{t+1}\big|\Xtil^0_t = x$
	varies with $t$. Consequently, the mutual dependencies between $\Xtil^0_1,\Xtil^0_2,\ldots $ could not be kept in the generation.
\end{remark}
\begin{remark}
	Due to the connection between $G(\cdot)$ and $g(\cdot)$ in (\ref{thm1-0}), $G(\cdot)$ could be considered as an extended function of $g(\cdot)$ that includes an extra forecasting lag variable $t$. 
\end{remark}
Given Theorem \ref{thm0} on the existence of $g$, now we consider the estimation of function $g(\cdot)$. 
For any given period $1,
\cdots,T$, we consider the following $f$-GAN \citep{nowozin2016f} type of min-max problem
for the iterative generation:
\bel{LGD1step}
&&(\ghat,\hhat) = \arg \mathop{\min}_{g\in \mathcal{G}_1} \mathop{\max}_{h\in\mathcal{H}_1} \widehat{\mathcal{L}}(g,h)
\cr && \widehat{\mathcal{L}}(g,h) = \frac1{T}\sum\limits_{t=0}^{T-1} \left[h(X_{t},g(\eta_{t},X_{t})) - f^*(h(X_{t},X_{t+1}))\right],
\eel
where $f^*(x)=\sup_y \{x\cdot y-f(y)\}$ is the convex conjugate of $f$, $\mathcal{G}_1$ and $\mathcal{H}_1$ are spaces of continuous and bounded functions. 

The $f$-divergence includes many commonly used measures of divergence, such as Kullback-Leibler (KL) divergence, $\chi^2$ divergence as special cases. In our analysis, we consider the general $f$ divergence. From a technique perspective, we assume that 
$f$ is a convex function and satisfies $f(1)=0$. We further assume that there exists constants $a>0$ and $0<b<1$ such that,
\bel{lm0-0}
f^{\prime\prime}(x+1)\ge \frac{a}{(1+bx)^3}, \ \forall \ x\ge-1.
\eel
The assumption (\ref{lm0-0}) is rather mild. For instance, the KL divergence, defined by $f(x)=x\log x$, meets the requirement of (\ref{lm0-0}) with $a=1$ and $b=1/3$. Similarly, the $\chi^2$ divergence, described by $f(x)=(x-1)^2$, satisfies (\ref{lm0-0}) with $a=1/4$ and $b=1/2$.

Now we define the  pseudo dimension and global bound of $\mathcal{G}_1$ and $\mathcal{H}_1$, which will be used in later analysis. Let $\mathcal{F}$ be a class of functions from $\mathcal{X}$ to $\mathbb{R}$. 
The pseudo dimension of $\mathcal{F}$, written as $\text{Pdim}_\mathcal{F}$, is the largest integer $m$ for which there exists $\left\{x_1,\cdots, x_m,y_1,\cdots,y_m\right\}\in \mathcal{X}^m \times \mathbb{R}^m$ such that for any $(b_1,\cdots, b_m)\in \left\{0,1\right\}^m$ there exists $f\in \mathcal{F}$ such that $\forall i: f(x_i)>y_i \Leftrightarrow b_i=1$. 
Note that this definition is adopted from \citet{bartlett2017nearlytight}. Furthermore, the global bound of $\mathcal{F}$ is defined as $B_\mathcal{F} := \mathop{\sup}_{f\in\mathcal{F}} \Vert f \Vert_\infty$.

The min-max problem (\ref{LGD1step}) is derived from the $f$-divergence between the pairs $(X_t,X_{t+1})$ and $(X_t,g(\eta_t,X_t))$. For any  two probability distributions with densities $p$ and $q$, let $D_f(q\|p) =\int f(\frac{q(z)}{p(z)})p(z) dz$ be their $f$-divergence. 
Denote the  $f$-divergence between $(X_t,X_{t+1})$ and $(X_t,g(\eta_t,X_t))$ as $\mathbb{L}_t(g)$:
\bel{1stepLG}
\mathbb{L}_t(g) = D_{f}(p_{X_t,g(\eta_t,X_t)}\Vert p_{X_t,X_{t+1}}).
\eel
A variational formulation of $f$-divergence \citep{keziou2003dual,nguyen2010estimating} is based on the Fenchel conjugate. Let
\bel{1stepLGD}
\mathcal{L}_t(g,h) = \mathbb{E}_{X_t,\eta_t} h(X_t,g(\eta_t,X_t)) - \mathbb{E}_{X_t,X_{t+1}} f^*(h(X_t,X_{t+1})),
\eel
Then we have $	\mathbb{L}_t(g) \ge \sup_{h\in \mathcal{H}_1} \mathcal{L}_t(g,h)$. The equality holds if and only if $f'(q_t/p_t)\in \mathcal{H}_1$, where $p_t$ and $q_t$ denote the distribution of $(X_t, X_{t+1})$ and  $\left(X_t, g(\eta_t,X_{t})\right)$, respectively.
If $f'(q_t/p_t)\in \mathcal{H}_1$ for $t=0\ldots, T-1$, then by Theorem \ref{thm0}, there exists a function $g$ such that 
\bes
\mathbb{L}_t(g)=\sup_{h\in \mathcal{H}_1} \mathcal{L}_t(g,h)=0, \ \  \forall t=0, 1\ldots, T-1.
\ees
Consequently, a function $g$ exists such that
\bes
\sup_{h\in \mathcal{H}_1} \E\widehat{\mathcal{L}}(g,h)=0.
\ees


For the $s$-step generation, we consider a similar min-max problem of the following form:
\bel{LhatGD}
&&(\widehat{G},\widehat{H}) =	\arg \mathop{\min}_{G\in\mathcal{G}} \mathop{\max}_{H\in \mathcal{H}} \widetilde{\mathcal{L}}(G,H)
\cr && 
\widetilde{\mathcal{L}}(G,H) = \frac{1}{|\Omega|}\sum\limits_{(t,s)\in \Omega} [H(X_{t},G(\eta_{t},X_{t}, s), s) - f^*(H(X_{t},X_{t+s},s))]
\cr && \Omega=[(t,s): t+s\le T, \ t\ge 1, \ 1\le s\le S].
\eel

Similar to $(\mathcal{G}_1, \mathcal{H}_1)$, here $\mathcal{G}$ and $\mathcal{H}$ are the spaces of continuous and bounded functions.
As the s-step generation allows us to generate outcomes after $s$-steps for an arbitrary $s$, the $ \widetilde{\mathcal{L}}(G,H)$ includes all the available pairs before the observation time $T$. As a comparison,  in iterative generation, the pairs are restricted to neighbors. The $s$-step generation is related to the following average of $f$-divergence a
\bel{LG}
\dot{\mathbb{L}}_t(G) = \frac1{S}\sum\limits_{s=1}^{S}D_{f}\left(p_{X_t,G(\eta_t,X_t,s)}\| p_{X_t,X_{t+s}}\right)
\eel
As in iterative generation, we also consider the following variational form
\bel{LGD}
\dot{\mathcal{L}}_t(G,H) =& S^{-1}\sum\limits_{s=1}^{S} \Big[\mathbb{E}_{X_t,\eta_t} H(X_t,G(\eta_t,X_t,s),s) \cr &- \E_{X_t,X_{t+s}} f^*(H(X_t,X_{t+s},s))\Big].
\eel
Similarly, we have $\dot{\mathbb{L}}(G) =\mathop{\sup}_{H} \dot{\mathcal{L}}(G,H)$.

On the other hand, the solution of s-step generation (\ref{LhatGD}) and iterative generation (\ref{LGD1step}) could also be connected as in the following proposition.
\begin{prop}\label{prop3}
	Let $(\Ghat,\Hhat)$ and $(\ghat,\hhat)$ be defined as in (\ref{LhatGD}) and (\ref{LGD1step}) respectively. Then,
	\begin{align*}
		&\Ghat(\cdot,\cdot,1)\equiv\ghat(\cdot,\cdot),\\
		&\Hhat(\cdot,\cdot,1)\equiv\hhat(\cdot,\cdot).
	\end{align*}
\end{prop}
Thus, the property (\ref{thm1-0}) on the target function $g$ and $G$ could also be inherited by the estimated solution $\hat{g}$ and $\hat{G}$.


%

\section{Convergence analysis for Lag-1 time series}\label{sec:3}

\subsection{General bounds for iterative generation}
\noindent In this section, we prove that the time series obtained by iterative generation matches the true distribution as in (\ref{joint1}) asymptotically. To achieve this goal, we impose the following condition on $X_t$.
%

\begin{asmp}\label{a2}
	The probability density funtion of $X_t$, denoted as $p_t$, converges in $L_1$, i.e., there exists a funtion $p_{\infty}$ such that:
	\bel{pdf converges}
	\int \vert p_t(x)-p_\infty (x)\vert dx \le \mathcal{O}(t^{-\alpha}),
	\eel
	where $\alpha>0 $ is a constant.
\end{asmp}

Assumption \ref{a2} requires that the density $X_t$ converges in $L_1$, and the convergence rate is controlled by $t^{-\alpha}$. 
Under many scenarios, such a rate is rather mild and can be achieved easily. For example, if we consider the following Gaussian distribution family, the convergence rate will be controlled by $e^{-ct}$, a smaller order of $t^{-\alpha}$. 
\begin{exm}\label{prop4new}
	Let $\left\{X_t\right\}$ be time series which satisfy:
	\bes
	X_0 \sim N(0, \Sigma_0),  \ \ 
	X_{t+1} = \phi_2 X_t + \phi_1 \xi_t, \ \ t=0,1,2,\ldots,
	\ees
	where $\phi_1, \phi_2\in \mathbb{R}^{p\times p}$, $\phi_2$ is symmetric matrix and its largest singular value is less than $1$. $\xi_t \overset{\text{i.i.d}}{\sim} N(0,I_p)$ is independent of $X_0$. Let $c = -2\log \sigma_{\text{max}}(\phi_2)$. Then there exists a density function $p_\infty$ such that
	\bel{gaussian rate}
	\int \vert p_{X_t}(x) - p_\infty(x)\vert dx = \mathcal{O}(e^{-ct}).
	\eel
\end{exm}

%

In classical time series analysis, stationary conditions are usually imposed for desired statistical properties. For example,   
A time series $\left\{Y_t\right\}$ is said to be strictly stationary or strongly stationary if
\bes
F_{Y_{t_1},\cdots,Y_{t_n}}(y_1,\cdots,y_n) = F_{Y_{t_1+\tau},\cdots,Y_{t_n+\tau}}(y_1,\cdots,y_n)
\ees
for all $t_1,\cdots,t_n,\tau\in\mathbb{N}$ and for all $n>0$, where $F_X(\cdot)$ is the distribution function of $X$. 
Clearly, the strictly stationary condition implies Assumption \ref{a2}. In other words, Assumption \ref{a2} is a much weaker condition, since the convergence condition is imposed only on the marginal densities, with no requirement placed on the joint distribution.

Given the convergence condition, we are ready to state the main theorem for the iterative generations. Let $\ghat$ be the solution to (\ref{LGD1step}) and $\Xhat_{T+s}$ be the generated sequence, i.e.,
\bes
\Xhat_T = X_T, \ \ \Xhat_{T+s} = \ghat(\eta_{T+s-1},\Xhat_{T+s-1}), \ \ s = 1,\ldots,S.
\ees 
Then, the joint density of $(\Xhat_T,\cdots,\Xhat_{T+S})$ converge to the corresponding truth in expectation as in  Theorem \ref{mainthm} below.

\begin{thm}\label{mainthm}
	(iterative generation) Let $X_0,X_1,\cdots,$ be a sequence of random variables which satisfy the Markov and Conditional invariance condition as in Theorem \ref{thm0}. Suppose Assumption \ref{a2} holds. Let $\ghat$ be the solution to the f-GAN problem (\ref{LGD1step}) with f satisfying (\ref{lm0-0}). Then,
	\bel{ineq2}
	\mathbb{E}_{(X_t,\eta_t)_{t=0}^T}\| p_{\Xhat_T,\cdots,\Xhat_{T+S}} - p_{X_T,\cdots,X_{T+S}}\|^2_{L_1} \le\underbrace{\Delta_1 + \Delta_2}_{\text{statistical error}}+  \underbrace{\Delta_3+ \Delta_4}_{\text{approximation error}},
	\eel
	where
	\begin{align*}
		&\Delta_1= \mathcal{O}( T^{-\frac{\alpha}{\alpha+1}}+T^{-2\alpha}),\\
		&\Delta_2 = \mathcal{O}\left(\sqrt{\frac{\text{Pdim}_{\mathcal{G}_1}\log (T\text{B}_{\mathcal{G}_1})}{T}}+\sqrt{\frac{\text{Pdim}_{\mathcal{H}_1}\log (T\text{B}_{\mathcal{H}_1})}{T}}\right),\\ 
		&\Delta_3 = \mathcal{O}(1)\cdot\mathbb{E}_{(X_t, \eta_t)_{t=0}^T} \left(\mathop{\sup}_{h} \mathcal{L}_T(\ghat,h) - \mathop{\sup}_{h\in \mathcal{H}_1} \mathcal{L}_T(\ghat,h)\right) ,\\
		&\Delta_4 = \mathcal{O}(1)\cdot \mathop{\inf}_{\bar{g}\in \mathcal{G}_1}\mathbb{L}_T(\bar{g}).
	\end{align*}
	Moreover, for the pairwise distribution, we have
	\bel{convergence2}
	\mathbb{E}_{(X_t, \eta_t)_{t=0}^T} \frac1S\sum\limits_{s=1}^{S} \Vert p_{\Xhat_T,\Xhat_{T+s}} - p_{X_T,X_{T+s}} \Vert_{L_1}^2 \le\underbrace{\Delta_1 + \Delta_2}_{\text{statistical error}}+  \underbrace{\Delta_3+ \Delta_4}_{\text{approximation error}}.
	\eel
\end{thm}

\begin{col}\label{col1}
	When $\left\{X_t\right\}$ follows a Gaussian distribution family as in Example 1, the statistical error $\Delta_1$ could be further optimized to
	\bes
	\Delta_1=\mathcal{O}\left(e^{-ct}+(1/T)\log T\right),
	\ees
	where $c$ is a certain constant. 
\end{col}

Theorem \ref{mainthm} establishes the convergence of joint and pairwise distribution for the iterative generations. The $\ell_1$ distance between $(X_T,\Xhat_{T+1},\cdots, \Xhat_{T+S})$ and the true sequence could be bounded by four terms, including the statistical errors $\Delta_1$, $\Delta_2$ and approximation errors  $\Delta_3$, $\Delta_4$. In particular, it is clear that $\Delta_1$ converge to 0 when $T\rightarrow \infty$. While for $\Delta_2$ to $\Delta_4$, we prove their converge in Section \ref{subsec3-3} when $\mathcal{H}_1$ and $\mathcal{G}_1$ are chosen to be spaces of deep neural networks. 
As a result of Theorem \ref{mainthm}, the main objective (\ref{joint1}) can be ensured asymptotically. In other words, the generated sequence $(X_T,\Xhat_{T+1},\cdots, \Xhat_{T+S})$ follows approximately the same distribution as the truth $(X_T,X_{T+1},\cdots,X_{T+S})$ with sufficient samples.

\begin{remark}\label{rmk-d1}
	The statistical error $\Delta_1$  depends on $T$ and $
	\alpha$, the convergence speed of $X_t$. Clearly, we may omit the term $\mathcal{O}(T^{-2\alpha})$ in $\Delta_1$ as $T^{-2\alpha} $ is of smaller order of $T^{-\frac{\alpha}{\alpha+1}}$. We include $T^{-2\alpha}$ in $\Delta_1$ because it controls the difference between joint distribution and pairwise distribution bound. See Proposition \ref{prop6} below. In addition, the term $\mathcal{O}(T^{-\frac{\alpha}{\alpha+1}})$ is obtained by estimating a carefully constructed quantity $d_s(G,H)$ in Proposition \ref{prop7} below.
\end{remark}


\begin{prop}\label{prop6} For any $S=1,2\ldots$,
	\bel{prop4-1-1}
	&&	\mathbb{E}_{(X_t,\eta_t)_{t=0}^T}\left\| p_{\Xhat_T,\cdots,\Xhat_{T+S}} - p_{X_T,\cdots,X_{T+S}}\right\|_{L_1}^2
	\cr&  \le & 
	2S^2\mathbb{E}_{(X_t, \eta_t)_{t=0}^T} \left\| p_{\Xtil_T,\Xtil_{T+1}} - p_{X_T,X_{T+1}} \right\|_{L_1}^2 + \mathcal{O}(T^{-2\alpha}).
	\eel
\end{prop}

\begin{prop}\label{prop7}
	For any $S=1,2\ldots$, define
	\bel{prop4-1-2}
	U(x,y,z,s) &=& H(x,G(z,x,s),s) - f^*(H(x,y,s)),
	\cr	d_s(G,H) &=& \mathbb{E}_{X_T,X_{T+s},\eta_T} U(X_T,X_{T+s},\eta_T,s) \cr &&- \frac1{T-s+1}\sum\limits_{t=0}^{T-s}\mathbb{E}_{X_t,X_{t+s},\eta_t}U(X_t,X_{t+s},\eta_t,s),
	\eel
	then we have
	\bel{prop4-1-3}
	\mathop{\sup}_{G\in \mathcal{G},H\in\mathcal{H}} \vert d_s(G,H)\vert \le \mathcal{O}(T^{-\frac\alpha{\alpha+1}}).
	\eel
\end{prop}

\begin{remark}\label{rmk-d2}
	The statistical error $\Delta_2$ depends only on the time period $T$ and the structure of function spaces $\mathcal{G}$ and $\mathcal{H}$. In subsection \ref{subsec3-3}, we show that $\Delta_2$ goes to 0 when  $\mathcal{G}$ and $\mathcal{H}$ are taken as neural network spaces of appropriate sizes. The $\Delta_2$ is obtained by estimating the Rademacher complexity of $\left\{b_{t}^s\right\}_{t=0}^{T-s}$ in Proposition \ref{prop3-1} below. Under the time series setting, $\left\{X_t, t=1,2\ldots\right\}$ are highly correlated. Conventional techniques to bound Rademacher complexity does not work. In our proof, we adopt a new technique introduced by \cite{mcdonald2017rademacher} that allows us to handle correlated variables. We defer to the supplementary material for more details.
\end{remark}

\begin{prop}\label{prop3-1}
	Let $\left\{\epsilon_t\right\}_{t\ge0}$ be the Rademacher random variables. For $1\le s\le S$, define
	\bel{rmk3-2}
	b_{t}^s(G,H) =&& H(X_t, G(\eta_t,X_t,s),s) - f^*(H(X_t,X_{t+s},s),s) 
	\cr &&- \mathbb{E}\big[H(X_t, G(\eta_t,X_t,s),s) - f^*(H(X_t,X_{t+s},s),s)\big].
	\eel
	Further let $	\mathcal{R}_s(\mathcal{G}\times\mathcal{H}) $ be the Rademacher complexity of $\left\{b_{t}^s\right\}_{t=0}^{T-s}$,
	\bel{prop4-1-1}
	\mathcal{R}_s(\mathcal{G}\times\mathcal{H}) = \mathbb{E} \mathop{\sup}_{G\in\mathcal{G},H\in\mathcal{H}} \left| \frac2{T-s+1}\sum\limits_{t=0}^{T-s} \epsilon_t b_{t}^s(G,H)\right|.
	\eel
	Then we have
	\bel{prop4-1-2}
	\mathbb{E} \mathop{\sup}_{G\in\mathcal{G},H\in\mathcal{H}} \left| \frac1{T-s+1}\sum\limits_{t=0}^{T-s} b_{t}^s(G,H)\right| \le \mathcal{R}_s(\mathcal{G}\times\mathcal{H}).
	\eel
	Moreover, $\mathcal{R}_s(\mathcal{G}\times\mathcal{H})$ could be further bounded using the pseudo dimension and global bound of $\mathcal{G}$ and $\mathcal{H}$ .
\end{prop}

%
%


\subsection{General bounds for $s$-step generation}
\noindent In this subsection, we provide theoretical guarantees for the $s$-step generation. Let $\Ghat$ be the solution to (\ref{LhatGD}) and $\Xtil_{T+s}$ be the generated sequence, i.e., 
\bes
\Xtil_{T+s} = \Ghat(\eta_{T},\Xtil_{T}, s), \ \ s = 1,\ldots,S.
\ees 
Now we show that the pairwise distance between $(\Xtil_T,\Xtil_{T+s})$ and $(X_T,X_{T+s})$ could be guaranteed as in Theorem \ref{thm2} below. 
\begin{thm}\label{thm2}
	($s$-step generation) 
	Let $X_0,X_1,\cdots,$ be a sequence of random variables which satisfy the Markov and Conditional invariance condition as in Theorem \ref{thm0}. Suppose Assumption \ref{a2} holds. Let $\Ghat$ be the solution to the f-GAN problem (\ref{LhatGD}) with f satisfying (\ref{lm0-0}). Then,
	\bel{couvergence}
	\mathbb{E}_{(X_t, \eta_t)_{t=0}^T}\frac1S\sum\limits_{s=1}^{S} \Vert p_{\Xtil_T,\Xtil_{T+s}} - p_{X_T,X_{T+s}} \Vert_{L_1}^2 \le\underbrace{\Deltatilde_1+ \Deltatilde_2}_{\text{statistical error}}+ \underbrace{\Deltatilde_3+ \Deltatilde_4}_{\text{approximation error}},
	\eel
	where
	\begin{align*}
		&\Deltatilde_1 =  \mathcal{O}( T^{-\frac{\alpha}{\alpha+1}}),\\
		&\Deltatilde_2 = \mathcal{O}\left(\sqrt{\frac{\text{Pdim}_{\mathcal{G}}\log (T\text{B}_{\mathcal{G}})}{T}}+\sqrt{\frac{\text{Pdim}_{\mathcal{H}}\log (T\text{B}_{\mathcal{H}})}{T}}\right),\\
		&\Deltatilde_3 = \mathcal{O}(1)\cdot\mathbb{E}_{(X_t, \eta_t)_{t=0}^T}\left(\mathop{\sup}_{H} \dot{\mathcal{L}}_T(\Ghat,H) - \mathop{\sup}_{H\in \mathcal{H}} \dot{\mathcal{L}}_T(\Ghat,H)\right),\\
		&\Deltatilde_4 = \mathcal{O}(1)\cdot\mathop{\inf}_{\bar{G}\in \mathcal{G}}\dot{\mathbb{L}}_T(\bar{G}).
	\end{align*}
	In particular, when $S=1$,
	\bel{convergence0}
	\mathbb{E}_{(X_t, \eta_t)_{t=0}^T}\Vert p_{\Xtil_T,\Xtil_{T+1}} - p_{X_T,X_{T+1}} \Vert_{L_1}^2\le\underbrace{\Deltatilde_1+\Delta_2}_{\text{statistical error}}+  \underbrace{\Delta_3+ \Delta_4}_{\text{approximation error}},
	\eel
	where $\Delta_2, \Delta_3$ and $\Delta_4$ are the quantities in Theorem \ref{mainthm}.
\end{thm}


Theorem \ref{thm2} demonstrates the convergence of pairwise distribution for the $s$-step generation. It states that the $\ell_1$ distribution distance between $(\Xtil_T,\Xtil_{T+s})$ and $(X_T,X_{T+s})$ for any $s$ can be bounded by the sum of two statistical errors and two approximation errors. In the following subsection, similar to Theorem \ref{mainthm}, we will show that all $\Deltatilde$s approach to 0 when $\mathcal{H}$ and $\mathcal{G}$ are chosen to be spaces of deep neural networks.
Upon comparing Theorem \ref{thm2} with  Theorem \ref{mainthm}, it can be observed that the term $T^{-2\alpha}$ in $\Delta_1$ does not appear in $\Deltatilde_1$. As discussed in Remark \ref{rmk-d1}, the term $T^{-2\alpha}$ controls the difference between joint and pairwise distribution, and it is no longer needed in Theorem \ref{thm2}.

\begin{remark}
	Unlike Theorem \ref{mainthm}, the convergence for joint distribution may not be guaranteed for s-step generated sequence $\big\{\Xtil_t, \ t=1,2\ldots \big\}$. As discussed in Theorem \ref{thm1}, there is no assurance regarding the existence of $G$ to attain the joint distribution match. The major issue for s-step generation is that $\Xhat_{T+s} \vert (\Xhat_{T+s-1}=x)$ varies with $t$. Thus the mutual dependencies between $\Xtil^0_1,\Xtil^0_2,\ldots $ could not be kept in the generation.
	As a comparison, in iterative generation, the joint distribution match could be achieved as the conditional distribution of adjacent generations does not vary with $t$, i.e.,
	\bes
	\Xhat_{T+s} \vert (\Xhat_{T+s-1}=x) \ \overset{\text{d}}{\equiv} \ \Xhat_{T+1} \vert (\Xhat_{T}=x).
	\ees
\end{remark}


%
%

\subsection{Analysis of deep neural network spaces}\label{subsec3-3}
\noindent Neural networks have been extensively studied in recent years due to its universal approximation power. 
In this subsection, we consider DNN to approximate the generator $g$ in our model. In particular, we show that both statistical and approximation errors converge to zero  when $\calG_1$, $\calH_1$, $\calG$, $\calH$ are taken to be the space of Rectified Linear Unit (ReLU) neural network functions. To avoid redundancy, we concentrate on the spaces of $\calG_1$ and $\calH_1$, with generalizations to $\calG$ and $\calH$ being straightforward.

Recall that the input $X_t$ and reference $\eta_t$ are of dimension $p$ and $m$, respectively. We consider the generator $g: \bbR^{p + m }  \rightarrow \bbR^{p}$ in the space of ReLU neural networks $\calG_1 := \calG_{\calD,\calW,\calK,\calB}$  with width $\mathcal{W}$, depth $\mathcal{D}$, size $\mathcal{K}$ and global bound $\mathcal{B}$. 
Specifically, let  $\omega_j$ denote the number of hidden units in layer $j$ with $\omega_0 = p+m$ being the dimension of input layer. Then the width $\mathcal{W} = \mathop{\max}_{0\le i\le \mathcal{D}}\left\{\omega_i\right\}$ is the maximum dimension, 
the depth $\mathcal{D}$ is the number of layers, the size $\mathcal{K} = \sum_{i=0}^{\mathcal{D}-1}  \omega_i\cdot (\omega_{i+1} + 1)$ is the total number of parameters, and the global bound satisfies $\|g\|_{\infty} \leq \mathcal{B}$ for all $g \in \mathcal{G}_1$. Similarly, we may define a ReLU network space for the discriminator $h$ as $\calH_1 :=\mathcal{H}_{\widetilde{\mathcal{D}},\widetilde{\mathcal{W}},\tilde{\mathcal{K}},\widetilde{\mathcal{B}}}$.


Then, by \citet{bartlett2017nearlytight}, we may bound the pseudo dimension of  $\calG_1$ (and $\calH_1$) and consequently $\Delta_2$ as in Proposition \ref{prop: delta1} below.


\begin{prop}\label{prop: delta1}
	Let $\calG_1 := \calG_{\calD,\calW,\calK,\calB}$  be the ReLU network space with width $\mathcal{W}$, depth $\mathcal{D}$, size $\mathcal{K}$ and global bound $\mathcal{B}$. Then we have
	\bes
	\text{Pdim}_{\mathcal{G}_1} = \mathcal{O}(\calD\mathcal{K}\log \mathcal{K}).
	\ees
	Consequently,
	\bes
	\Delta_2 = \mathcal{O}\left(\sqrt{\frac{\mathcal{D}\mathcal{K}\log\mathcal{K}\log(T\mathcal{B})}{T}} + \sqrt{\frac{\tilde{\mathcal{D}}\tilde{\mathcal{K}}\log\tilde{\mathcal{K}}\log(T\tilde{\mathcal{B}})}{T}}\right).
	\ees
\end{prop}

By Proposition \ref{prop: delta1}, it is clear that $\Delta_2$ goes to 0 with appropriate size of ReLU network spaces, e.g. $\mathcal{D}\mathcal{K}\log\mathcal{K}\log(T\mathcal{B})$ and $\tilde{\mathcal{D}}\tilde{\mathcal{K}}\log\tilde{\mathcal{K}}\log(T\tilde{\mathcal{B}})$ are of smaller order of $T$. Moreover, as $\Delta_2\rightarrow 0$ when $T\rightarrow \infty$ regardless of network structure, we can conclude that the statistical error in Theorem \ref{mainthm} converge to 0.


Now we consider the approximation errors $\Delta_3$ and $\Delta_4$. The approximation power of DNN has been intensively studied in the literature under different conditions, such as smoothness assumptions. 
For instance, the early work by \citet{stone1982optimal} established the optimal minimax rate of convergence for estimating a $(\beta, C)$-smooth function. While more recently, \citet{yarotsky2017error} and \citet{lu2020deep} considered target functions with continuous $\beta$-th derivatives. \citet{jiao2021deep} assumed $\beta$-Hölder smooth functions with $\beta>1$. Moreover,  studies including \cite{shen2021deep}, \citet{schmidt2020nonparametric}, and \citet{bauer2019deep}, have sought to enhance the convergence rate by assuming that the target function possesses certain compositional structure. Here, we adopt Theorem 4.3 in \citet{shen2019deep} and show that the approximation error in Theorem \ref{mainthm} converge to 0 with a particular structure of neural networks.



\begin{prop}\label{prop: delta3}
	Let $\calG_1 := \calG_{\calD,\calW,\calK,\calB}$  be a ReLU network space with depth $\calD = 12\log T  + 14 + 2(p+m+1)$ and width $ \calW = 3^{p+m+4} \max\{(p+m+1)\lfloor (T ^{\frac{p+m+1}{2(3+p+m)}}/\log T)^{\frac{1}{p+m+1}}\rfloor,T ^{\frac{p+m+1}{2(3+p+m)}}/\log T + 1\}$. Further let  $\calH_1 :=\mathcal{H}_{\widetilde{\mathcal{D}},\widetilde{\mathcal{W}},\tilde{\mathcal{K}},\widetilde{\mathcal{B}}}$  be a ReLU network with depth $\widetilde{\calD} = 12\log T  + 14 + 2(2p+1)$ and width $ \widetilde{\calW} = 3^{2p+4} \max\{(2p+1)\lfloor (T ^{\frac{2p+1}{2(2p+3)}}/\log T)^{\frac{1}{2p+1}}\rfloor,T ^{\frac{2p+1}{2(2p+3)}}/\log T + 1\}$.
	Then as $T\rightarrow \infty$, we have 
	\bes
	&& \Delta_3 = \mathcal{O}(1)\cdot\mathbb{E}_{(X_t, \eta_t)_{t=0}^T} \left(\mathop{\sup}_{h} \mathcal{L}_T(\ghat,h) - \mathop{\sup}_{h\in \mathcal{H}_1} \mathcal{L}_T(\ghat,h)\right) \rightarrow 0,
	\cr && \Delta_4 = \mathcal{O}(1)\cdot \mathop{\inf}_{\bar{g}\in \mathcal{G}_1}\mathbb{L}_T(\bar{g})
	\rightarrow 0.
	\ees
\end{prop}

\section{Generalizations to lag-k time series} \label{sec:4}
\noindent In this section, we generalize the lag-1 time series studied in Section \ref{sec:2} and \ref{sec:3} to a lag-k setting.
Specifically, we consider a time series $\left\{X_t \in\mathbb{R}^p, t=1,2,\ldots \right\}$ that satisfies the following lag-k Markov assumption
\bel{lagk1}
X_t\vert X_{t-1},\ldots,X_0 \ \overset{\text{d}}{=} \ X_t\vert X_{t-1},\ldots,X_{t-k}.
\eel
Moreover, assume that $\left\{X_t \in\mathbb{R}^p, t=1,2,\ldots \right\}$  is conditionally invariant
\bel{lagk2}
p_{X_t\vert X_{t-1},\cdots,X_{t-k}}(x_k\vert x_{k-1},\cdots,x_0) = p_{X_k\vert X_{k-1},\cdots,X_0}(x_k\vert x_{k-1},\cdots,x_0), \ \ \forall t \ge k.
\eel
In other words, the conditional density function of $X_t\vert X_{t-1},\cdots, X_{t-k}$ does not depend on $t$.

Given a lag-k time series $\left\{X_t\right\}$, 
we aim to generate a sequence $(\Xhat_{T+1},\ldots, \Xhat_{T+S})$ that not only follows the same joint distribution as $(X_{T+1},\ldots, X_{T+S})$,  but also maintains the dependencies between $(\Xhat_{T+1},\ldots, \Xhat_{T+S})$ and $(X_{T-k+1},\ldots, X_T)$.
In other words, we aim to achieve
\bel{lagkjoint1}
(\Xhat_{T-k+1},\ldots,\Xhat_{T},\Xhat_{T+1},\ldots,\Xhat_{T+S})\overset{\text{d}}{=} (X_{T-k+1},\ldots,X_{T},X_{T+1},\ldots,X_{T+S}).
\eel
we show that such generation is possible by the following iterative generation
\bes
&&(\Xhat_{T-k+1},\cdots,\Xhat_{T}) = (X_{T-k+1},\cdots,X_T),
\cr && \Xhat_{T+s} = g(\eta_{T-k+s}, \Xhat_{T-k+s},\cdots,\Xhat_{T-1+s}),\ \ 1\le s\le S.
\ees
where $g$ is the target function to be estimated and $\eta_t$ is i.i.d. Gaussian vectors of dimension $m$. Moreover, we may also consider the $s$-step generation:
\bes
\Xtil_{T+s}= G\left(\eta_{T-k+1}, X_{T-k+1},\cdots,X_{T},s\right).
\ees
The following proposition suggests that, analogous to the lag-1 case, there exist a function $g$ for iterative generation to achieve the joint distribution matching. Furthermore, for the s-step generation, a function $G$ exists to attain the marginal distribution matching.

\begin{prop}\label{prop4}
	Let $\left\{X_t\right\}$ satisfies the lag-k Markov property (\ref{lagk1}) and conditional invariance condition (\ref{lagk2}). Let $(\Xhat_0^0,\cdots,\Xhat_{k-1}^0)\overset{\text{d}}{=}(X_0,\cdots,X_{k-1})$ and $\eta_0,\eta_1,\cdots$ be independent $m$-dimensional Gaussian vectors which are independent of $(\Xhat_0,\cdots,\Xhat_{k-1})$. Then for iterative generation, there exists a measurable function $g$ such that the sequence
	\bel{prop4-1}
	\left\{\Xhat_t^0: \Xhat_t^0 = g(\eta_{t-k}, \Xhat_{t-k}^0,\cdots,\Xhat_{t-1}^0)\right\}
	\eel
	satisfies that for any $s\ge1$,
	\bel{prop4-2}
	(\Xhat_0^0,\cdots, \Xhat_s^0) \overset{\text{d}}{=} (X_0,\cdots, X_s).
	\eel
	Moreover, for $s$-step generation, the sequence
	\bel{prop4-3}
	\left\{\Xtil_{t+k-1}^0: \Xtil_{t+k-1}^0 = G(\eta_0, X_0,\cdots,X_{k-1},t)\right\}
	\eel
	satisfies
	\bel{prop4-4}
	\Xtil_{t+k-1}^0 \overset{\text{d}}{=} X_{t+k-1}.
	\eel
\end{prop}


Now we consider the estimation of $g$ and $G$ in lag-k time series. For any sequence $\left\{A_t\right\}$ and positive integer $u\le v$, denote $A_{[u,v]}$ as the set $(A_u, A_{u+1},\cdots,A_v)$.
Then we consider the following min-max problem for the estimation of s-step generation:
\bel{op3}
(\Ghat, \Hhat) &=& \arg \mathop{\min}_{G\in\mathcal{G}} \mathop{\max}_{H\in\mathcal{H}} \widetilde{\mathcal{L}}(G,H),
\cr \widetilde{\mathcal{L}}(G,H) &=& \frac1{\vert\Omega\vert} \sum\limits_{(t,s)\in\Omega} \big[H(X_{[t-k+1,t]},G(\eta_{t-k+1},X_{[t-k+1,t]},s),s) \cr && \ \ \ \ \ \ \ \ \ \ \ \ \ \ - f^*(H(X_{[t-k+1,t]},X_{t+s},s))\big],
\cr \Omega &=& [(t,s): t+h\le T, t\ge k, 1\le s\le S],
\eel
where $\mathcal{G}$, $\mathcal{H}$ are the spaces of continuous and bounded functions. As in the lag-1 case, the generator $g$ and discriminator $h$ in the lag-k iterative generation can be obtained by letting
\bes
\ghat(\cdot,\cdot)=\Ghat(\cdot,\cdot,1), \ \ \ \hhat(\cdot,\cdot)=\Hhat(\cdot,\cdot,1).
\ees

For lag-k time series, we impose the following condition  analogous to Assumption \ref{a2} in Section \ref{sec:3}.
\begin{asmp}\label{a5}
	The probability density funtion of $X_{[t,t+k-1]}$, denoted by $p_{t,k}$, converges in $L_1$, i.e., there exists a funtion $p_{\infty,k}$ such that:
	\bel{pdf converges}
	\int \left| p_{t,k}(x_1,\cdots,x_k)-p_{\infty,k} (x_1,\cdots,x_k)\right| dx_{[1,k]} \le \mathcal{O}(t^{-\alpha})
	\eel
	where $\alpha>0 $ is certain positive constant.
\end{asmp}
Given Assumption \ref{a5}, we can derive theoretical guarantees for the distribution matching of lag-k time series. 
Now let $\Xhat_T$ be the iteratively generated sequence, i.e.,
\begin{align}\label{lagk3}
	&\Xhat_{T-j} = X_{T-j}, \ \ j = 0 ,\cdots, k-1, \cr
	&\Xhat_{T+s} = \ghat(\eta_{T+s-k},\Xhat_{T+s-k},\cdots,\Xhat_{T+s-1}), \ \ s = 1,\cdots, S.
\end{align}
Further let $\Xtil_T$ be the s-step generated sequence, i.e., 
\begin{align}\label{lagk4}
	&\Xtil_{T-j} = X_{T-j}, \ \ j = 0,\cdots, k-1, \\
	&\Xtil_{T+s} = \Ghat(\eta_{T-k+1}, \Xtil_{T-k+1},\cdots,\Xtil_{T},s), \ \ s = 1,\cdots, S.
\end{align}
Then we have the following convergence theorem for iterative and s-step generated sequences. 
\begin{thm}\label{thm3}
	Let $\left\{X_t\right\}$ satisfies the lag-k Markov property (\ref{lagk1}) and conditional invariance condition (\ref{lagk2}).
	Suppose Assumption \ref{a5} holds. Let $\Ghat$ be the solution to the f-GAN problem (\ref{op3}) with f satisfying (\ref{lm0-0}).
	Then,  for the iterative generations $\Xhat_{T+s}$ in (\ref{lagk3}), we have
	\bel{thm3-1}
	\mathbb{E}_{(X_t,\eta_t)_{t=0}^T}\left\| p_{\Xhat_{[T-k+1,T+S]}} - p_{X_{[T-k+1,T+S]}}\right\|^2_{L_1} \le \underbrace{\Deltabar_1 + \Deltabar_2}_{\text{statistical err}}+ \underbrace{\Deltabar_3+ \Deltabar_4}_{\text{approximation err}},
	\eel
	and
	\bel{thm3-2}
	\mathbb{E}_{(X_t, \eta_t)_{t=0}^T} \frac1S\sum\limits_{s=1}^{S} \left\| p_{\Xhat_{[T-k+1,T]},\Xhat_{T+s}} - p_{X_{[T-k+1,T]},X_{T+s}} \right\|_{L_1}^2 \le \underbrace{\Deltabar_1 + \Deltabar_2}_{\text{statistical err}}+\underbrace{\Deltabar_3+ \Deltabar_4}_{\text{approximation err}},
	\eel
	where 
	\begin{align*}
		&\Deltabar_1 =\mathcal{O}( T^{-\frac{\alpha}{\alpha+1}}+T^{-2\alpha}), \\
		&\Deltabar_2 = \mathcal{O}\left(\sqrt{\frac{\text{Pdim}_{\mathcal{G}_1}\log (T\text{B}_{\mathcal{G}_1})}{T}}+\sqrt{\frac{\text{Pdim}_{\mathcal{H}_1}\log (T\text{B}_{\mathcal{H}_1})}{T}}\right),	\\
		&\Deltabar_3 = \mathcal{O}(1)\cdot\mathbb{E}_{(X_t, \eta_t)_{t=0}^T}(\mathop{\sup}_{h} \mathcal{L}_T(\ghat,h) - \mathop{\sup}_{h\in \mathcal{H}_1} \mathcal{L}_T(\ghat,h)),\\
		&\Deltabar_4 = \mathcal{O}(1)\cdot\mathop{\inf}_{\bar{g}\in \mathcal{G}_1}\mathbb{L}_T(\bar{g}).
	\end{align*}
	Moreover, for $s$-step generations $\Xtil_{T+s}$ in (\ref{lagk4}), we have
	\bel{thm3-3}
	\mathbb{E}_{(X_t, \eta_t)_{t=0}^T}\frac1S\sum\limits_{s=1}^{S} \left\|p_{\Xtil_{[T-k+1,T]},\Xtil_{T+s}} - p_{X_{[T-k+1,T]},X_{T+s}} \right\|_{L_1}^2 \le\underbrace{\breve{\Delta}_1+ \breve{\Delta}_2}_{\text{statistical err}} + \underbrace{\breve{\Delta}_3+ \breve{\Delta}_4}_{\text{approximation err}}
	\eel
	In particular, when $S=1$,
	\bel{thm3-4}
	\mathbb{E}_{(X_t, \eta_t)_{t=0}^T}\left\| p_{\Xtil_{[T-k+1,T+1]}} - p_{X_{[T-k+1,T+1]}} \right\|_{L_1}^2 \le  \underbrace{\breve{\Delta}_1+\Deltabar_2}_{\text{statistical err}}+\underbrace{\Deltabar_3+ \Deltabar_4}_{\text{approximation err}}.
	\eel
	where
	\begin{align*}
		&\breve{\Delta}_1 =  \mathcal{O}( T^{-\frac{\alpha}{\alpha+1}}),\\
		&\breve{\Delta}_2 = \mathcal{O}\left(\sqrt{\frac{\text{Pdim}_{\mathcal{G}}\log (T\text{B}_{\mathcal{G}})}{T}}+\sqrt{\frac{\text{Pdim}_{\mathcal{H}}\log (T\text{B}_{\mathcal{H}})}{T}}\right),\\
		&\breve{\Delta}_3 =\mathcal{O} (1)\cdot \mathbb{E}_{(X_t, \eta_t)_{t=0}^T}(\mathop{\sup}_{H} \dot{\mathcal{L}}_T(\Ghat,H) - \mathop{\sup}_{H\in \mathcal{H}} \dot{\mathcal{L}}_T(\Ghat,H)),\\
		&\breve{\Delta}_4 = \mathcal{O} (1)\cdot \mathop{\inf}_{\bar{G}\in \mathcal{G}}\dot{\mathbb{L}}_T(\bar{G}).
	\end{align*}
\end{thm}

\begin{remark}
	Analogous to the lag-1 case, the $\mathcal{L}_T$,  $\mathbb{L}_T$, $\dot{\mathcal{L}}_T$, and $\dot{\mathbb{L}}_T$ in the lag-k time series are defined as below:
	\bel{thm3-defn}
	\mathbb{L}_T(g) &=& D_f\left(p_{X_{[T-k+1,T]},g\left(\eta_{T-k+1},X_{[T-k+1,T]}\right)}\Vert p_{X_{[T-k+1,T+1]}}\right),
	\cr\mathcal{L}_T(g,h) &=& \mathbb{E}_{X_{[T-k+1,T]},\eta_{T-k+1}}h\left(X_{[T-k+1,T]},g(\eta_{T-k+1},X_{[T-k+1,T]})\right)
	\cr &&  - \mathbb{E}_{X_{[T-k+1,T+1]}} f^*\left(h(X_{[T-k+1,T+1]})\right),
	\cr  \dot{\mathbb{L}}_T(G) &=& \frac1S \sum_{s=1}^S D_f\left(p_{X_{[T-k+1,T]},G(\eta_{T-k+1},X_{[T-k+1,T]},s)}\Vert p_{X_{[T-k+1,T]}, X_{T+s}}\right),
	\cr \dot{\mathcal{L}}_T(G,H) &=& \frac1{S}\sum\limits_{s=1}^{S}\Big[\mathbb{E}_{X_{[T-k+1,T]},\eta_{T-k+1}}H\left(X_{[T-k+1,T]},G(\eta_{T-k+1},X_{[T-k+1,T]},s),s\right)
	\cr && - \mathbb{E}_{X_{[T-k+1,T]},X_{T+s}} f^*(H(X_{[T-k+1,T]},X_{T+s},s))\Big].
	\eel
	When $\mathcal{H}$ and $\mathcal{G}$ are approximated by appropriate deep neural networks, the $\Deltabar_1$ to $\Deltabar_4$ and $\Deltabreve_1$ to $\Deltabreve_4$ will all converge to 0. Consequently, the joint distribution matching for the iterative generation and pairwise distribution matching for the s-step generation could be guaranteed in lag-k time series.
\end{remark}
\section{Further generalizations to panel data}\label{sec:5}
\noindent In this section, we extend our analysis for image time series to a panel data setting. In particular, we consider a scenario with $n$ subjects, and for each subject $i=1,2,\cdots, n$, we observe a sequence of images $\left\{X_{i,t}, t=1,2,\ldots, T_i\right\}$. Here we allow the time series length $T_i$ for each subject to be different.  Clearly,  this type of setting is frequently encountered when analyzing medical image data. Our objective is to generate images for each subject at future time points.

In the panel data setting, we assume that $\left\{X_{i,t}, t=1,2,\ldots, T_i\right\}$ satisfies the following Markov condition for all subject
\bel{multi1}
X_{i,t}\vert X_{i,t-1},\cdots, X_{i,0} \overset{\text{d}}{=} X_{i,t}\vert X_{i,t-1}, \ \   i=[n], \ t = [T_i].
\eel
We further assume the following invariance condition 
\bel{multi2}
p_{X_{i,t}\vert X_{i,t-1}}(x\vert y) = p_{X_{1,1}\vert X_{1,0}}(x\vert y), \ \   i=[n], \ t = [T_i].
\eel
In other words, we assume the same conditional distribution for different subjects $i$ and time point $t$.


Similar to previous sections, we aim to find a common function $g$ such that for all subjects $i=1,2,\cdots, n$, the generated sequence
\bel{sec5-def joint}
\Xhat_{i,T_i} = X_{i,T_i}, \ \ \ \Xhat_{i,T_i+s} = g(\eta_{T_i+s-1},\Xhat_{i,T_i+s-1}), \ \ 1\le s \le S
\eel
achieves distribution matching
\bel{sec5-joint}
(\Xhat_{i,T_i},\Xhat_{i,T_i+1},\cdots, \Xhat_{i,T_i+S}) \sim (X_{i,T_i},X_{i,T_i+1},\cdots,X_{i,T_i+S}).
\eel
By Theorem \ref{thm0}, such a function $g$ clearly exist. To estimate $g$, we consider the following min-max problem.
\bel{op5}
&& (\widehat{g},\widehat{h}) = \arg \mathop{\min}_{g\in\mathcal{G}} \mathop{\max}_{h\in\mathcal{H}} \widehat{\mathcal{L}}(g,h),
\cr && \widehat{\mathcal{L}}(g,h) = \frac1{n} \sum\limits_{i=1}^n \frac1{T_i}\sum\limits_{t=0}^{T_i-1}\Big[h(X_{i,t},g(\eta_t,X_{i,t}))-f^*(h(X_{i,t},X_{i,t+1}))\Big],
\eel
where as before, $\mathcal{G}, \mathcal{H}$ are spaces of continuous and uniformly bounded functions. To prove the convergence of the generated sequence, we consider two different settings: 1) $T_i$ approaches infinity, while $n$ may either go to infinity or be finite; 2) $T_i$ is finite, while $n$ approaches infinity.

%
%
\subsection{Convergence analysis for $T_i\to \infty$}
\noindent In this subsection, we consider the case that $\mathop{\min}_{1\le i\le n} \left\{T_i\right\} \to \infty$, while $n$ may either go to infinity or be finite. We consider the following sequences
\bel{multi3}
\widehat{X}_{i,T_i} = X_{i,T_i}, \ \ \widehat{X}_{i,T_i+s} = \widehat{g}(\eta_{T_i+s-1}, \widehat{X}_{i, T_i+s-1}), \ \ i = [n],\  s = [S].
\eel

Now we are ready to present the convergence theorem for the generated sequences.
\begin{thm}\label{thm13}
	Suppose $\left\{X_{i,t} \right\}$ satisfies the Markov property (\ref{multi1}) and conditional invariance condition (\ref{multi2}).
	Suppose $\left\{X_{i,t}, \ t=[T_i]\right\}$ satisfies Assumption \ref{a2} for each subject $i$.   Let $\ghat$ be the solution to the f-GAN problem (\ref{op5}) with f satisfying (\ref{lm0-0}). Then,
	\bel{thm13-1}
	\mathbb{E}_{\{\eta_t, X_{i,t}\}}\frac1n\sum_{i=1}^n\left\| p_{\widehat{X}_{i,T_i},\cdots,\widehat{X}_{i,T_i+S}} - p_{X_{i,T_i},\cdots,X_{i, T_i+S}} \right\|_{L_1}^2 \le \underbrace{\dot{\Delta}_1 + \dot{\Delta}_2}_{\text{statistical error}} + \underbrace{\dot{\Delta}_3 + \dot{\Delta}_4}_{\text{approximation error}}
	\eel
	where
	\bes
	&&\dot{\Delta}_1= \mathcal{O}\left(\frac1n \sum\limits_{i=1}^n\left[T_i^{-\frac\alpha{\alpha+1}} + T_i^{-2\alpha}\right]\right),\cr
	&&\dot{\Delta}_2 = \mathcal{O}\left(\frac1n \sum\limits_{i=1}^n \left[\sqrt{\frac{\text{Pdim}_{\mathcal{G}}\log (T_i\text{B}_{\mathcal{G}})}{T_i}}+\sqrt{\frac{\text{Pdim}_{\mathcal{H}}\log (T_i\text{B}_{\mathcal{H}})}{T_i}}\right]\right),
	\cr &&\dot{\Delta}_3 = \mathcal{O}(1)\cdot\mathbb{E}(\mathop{\sup}_{h} \mathcal{L}_{(n)}(\widehat{g},h)- \mathop{\sup}_{h\in\mathcal{H}}\mathcal{L}_{(n)}(\widehat{g},h)),
	\cr &&\dot{\Delta}_4 = \mathcal{O}(1)\cdot\mathop{\inf}_{\bar{g}\in\mathcal{G}} \mathbb{L}_{(n)}(\bar{g}).
	\ees
	Here  $\mathcal{L}_{(n)}(\cdot,\cdot)$ and $\mathbb{L}_{(n)}(\cdot,\cdot)$ are defined as
	\bes
	&&\mathcal{L}_{(n)}(g,h) = \frac1n \sum\limits_{i=1}^n \big[\mathbb{E}_{\eta_{T_i},X_{i,T_i}} h(X_{i,T_i}, g(\eta_{T_i},X_{i,T_i})) - \mathbb{E}_{X_{i,T_i},X_{i,T_i+1}} f^*(h(X_{i,T_i}, X_{i,T_i+1}))\big]
	\cr && \mathbb{L}_{(n)}(g) = \frac1n \sum\limits_{i=1}^n D_f(p_{X_{i,T_i},g(\eta_{T_i},X_{i,T_i})}\Vert p_{X_{i,T_i},X_{i,T_i+1}}).
	\ees
\end{thm}

Whether $n$ is finite or approaches infinity, Theorem \ref{thm13} demonstrates the convergence of the generated sequence when $T_i\rightarrow \infty$. We shall note that the usual independence assumption between subjects is not necessary in Theorem \ref{thm13}. This implies that convergence can be guaranteed even when the observations are dependent.


\subsection{Convergence analysis for $n\to\infty$ and $T$ is finite }
\noindent In this subsection, we consider the case that $n\to \infty$, while $T_i$ is finite. Without loss of generality, we assume that $T_1 = T_2 = \cdots = T_n \triangleq T$. In addition, the following assumption is needed in our analysis.
\begin{asmp}\label{a6}
	For all $i =1,\ldots, n$, the starting point $X_{i,0}$ follows the same distribution as $X_0$, i.e.,
	\bel{a5-1}
	X_{1,0} \overset{\text{d}}{=} \cdots \overset{\text{d}}{=} X_{n,0} \overset{\text{d}}{=} X_0.
	\eel
\end{asmp}

By combining Assumption \ref{a6} with the Markov and conditional invariance conditions, we could have that the sequences $(X_{i,0},\cdots, X_{i,T})$ for all $i=1,\ldots, n$ follow the same joint distribution. Consequently, we can reach the following convergence theorem.

\begin{thm}\label{thm14}
	Suppose $\left\{X_{i,t} \right\}$ satisfies Assumption \ref{a6}, the Markov property (\ref{multi1}) and conditional invariance condition (\ref{multi2}). Further assume	$\int \frac{p_{X_{T+s}}^2(x)}{p_{X_{T-1}}(x)} dx < \infty$ for all $s=[S]$. 
	Then, if the sequences $(X_{i,0},\cdots, X_{i,T})$  are independent across samples $i=1,\ldots, n$, we have for all $i$,
	\bel{thm14-0}
	\mathbb{E} \left\| p_{\Xhat_{i,T},\cdots,\Xhat_{i,T+S}} - p_{X_T,\cdots, X_{T+S}}\right\|_{L_1}^2 \le \underbrace{\ddot{\Delta}_1}_{\text{statistical error}} + \underbrace{\ddot{\Delta}_2 + \ddot{\Delta}_3}_{\text{approximation error}},
	\eel
	where
	\bes
	&& \ddot{\Delta}_1 = \mathcal{O}\left(\sqrt{\frac{\text{Pdim}_{\mathcal{G}}\log (n\text{B}_{\mathcal{G}})}{n}}+\sqrt{\frac{\text{Pdim}_{\mathcal{H}}\log (n\text{B}_{\mathcal{H}})}{n}}\right),
	\cr &&\ddot{\Delta}_2 = \mathcal{O}(1)\cdot\mathbb{E}(\mathop{\sup}_{h} \dot{\mathcal{L}}_{(T)}(\widehat{g},h)- \mathop{\sup}_{h\in\mathcal{H}}\dot{\mathcal{L}}_{(T)}(\widehat{g},h)),
	\cr &&\ddot{\Delta}_3 = \mathcal{O}(1)\cdot \mathop{\inf}_{\bar{g}\in\mathcal{G}} \dot{\mathbb{L}}_{(T)}(\bar{g}).
	\ees
	Here  $\dot{\mathcal{L}}_{(T)}(\cdot,\cdot)$ and $\dot{\mathbb{L}}_{(T)}(\cdot,\cdot)$ are defined as
	\bes
	&&\dot{\mathcal{L}}_{(T)}(g,h)=\frac1T \sum\limits_{t=0}^{T-1} \big[\mathbb{E}_{\eta_t,X_{i,t}}h(X_{i,t},g(\eta_t,X_{i,t}))-\mathbb{E}_{X_{i,t},X_{i,t+1}}f^*(h(X_{i,t},X_{i,t+1}))\big],
	\cr &&\dot{\mathbb{L}}_{(T)}(g) = \frac1T\sum\limits_{t=0}^{T-1} D_f(p_{X_{i,t},g(\eta_t,X_{i,t})}\Vert p_{X_{i,t},X_{i,t+1}}).
	\ees
	Moreover,
	\bel{thm14-1}
	\mathbb{E}\frac1{T}\sum\limits_{t=0}^{T-1} \left\| p_{X_{i,t}, \widehat{g}(\eta_t, X_{i,t})} - p_{X_{i,t},X_{i,t+1}}\right\|_{L_1}^2 \le \underbrace{\ddot{\Delta}_1}_{\text{statistical error}} + \underbrace{\ddot{\Delta}_2 + \ddot{\Delta}_3}_{\text{approximation error}}.
	\eel
\end{thm}

\begin{remark}
	The independence assumption is necessary for convergence in Theorem \ref{thm14}, whereas it is not required in Theorem \ref{thm13}. In a panel data setting where the time series length $T$ is finite, we cannot rely on a sufficiently long time series to achieve convergence. In particular, the Proposition \ref{prop6} could no longer be employed to control the difference between joint and pairwise distribution bounds. Thus the proof for Theorem \ref{thm14} differs significantly from previous sections. 
\end{remark}

\section{Simulation studies}\label{sec:simulation}
\noindent In this section, we conduct comprehensive simulation studies to assess the performance of our generations. We begin in  Section \ref{sim1} to consider the generation of a single image time series, then in Section \ref{sim2} generalize to the panel data scenario.


\subsection{Study I: Single Time Series}\label{sim1}
\noindent We consider matrix valued time series to mimic the setting of real image data. Specifically, we consider the following three cases:

\begin{itemize}
	\item[] Case 1. Lag-1 Linear 
	\bes
	X_{t+1} = \phi_1  X_{t}  + \phi_e E_{t+1},
	\ees
	\item[] Case 2. Lag-1 Nonlinear
	\bes
	X_{t+1} = \phi_1 \sin X_{t} ^ \T + \phi_e E_{t+1},
	\ees
	\item[] Case 3. Lag-3 Nonlinear
	\bes
	X_{t+1} =  \phi_1 \cos (X_{t}^\T X_{t-2} X_{t}^\T)+ \phi_2 \sqrt{\max\{0, X_{t-1}^\T\}} 
	+  \phi_e E_{t+1}.
	\ees
\end{itemize}
Here $X_{t}\in\mathbb{R}^{p_1\times p_2}$ and $E_{t+1}\in\mathbb{R}^{p_1\times p_2}$ represent target image and independent noise matrices, respectively, with the image size fixed to be $(p_1,p_2) = (32,32)$. In all three cases, we let the noise matrix $E_{t+1}$ consists of i.i.d. standard normal entries. Moreover, the initialization $X_0$ in Cases 1 and 2 and $X_0$, $X_1$, and $X_2$ in Case 3 are also taken to the matrices with i.i.d. standard normal entries.
It is worth noting that under Case 1 setting,  $X_t$ is column-wise independent, and each column of $X_t$ converges to $\calN(0, \Sigma_\infty)$, where $\Sigma_\infty$ satisfies
$\Sigma_\infty  = \phi_1 \Sigma_\infty \phi_1^\T + \phi_e \phi_e^\T$.
In this simulation, we let $\phi_1\in\mathbb{R}^{32\times 32}$ be a normalized Gaussian matrix with the largest eigenvalue modulus less than 1, $\phi_e\in\mathbb{R}^{32\times 32}$ is a fixed matrix with block shape patterns. Given $\phi_1$ and $\phi_e$, then $\Sigma_\infty$ could be solved easily.  We plot the $\phi_1$, $\phi_e$ and $\Sigma_\infty$ in the supplementary material. Moreover, in Case 2, we set $\phi_1$, $\phi_e\in\mathbb{R}^{32\times 32}$ in a similar manner as Case 1, but certainly, the normal convergence would no longer hold. While in Case 3, we let both $\phi_1$ and $\phi_2\in\mathbb{R}^{32\times 32}$ be normalized Gaussian matrices,  and fix $\phi_e\in\mathbb{R}^{32\times 32}$ as before.

We consider the time series with two different lengths, $T=1000$ and $T=5000$ for training. We set the horizon $S=3$, i.e., generate images in a maximum of $S=3$ steps. We aim to generate a total of 
$T_{new}=500$ image points in a ``rolling forecasting'' style. Specifically, for any $s=1,\ldots, D$, we let the $s$-step generation (denoted as ``s-step GTS'') be
\bes
\Xtil_{T+t_{new}}^{(j)}(s)=\Ghat\left(X_{T+t_{new}-s}, \eta^{(j)}_{T+t_{new}-s}, s\right), 
\ees
where $t_{new}\le T_{new}$ and superscript $j$ indicting the j-th generation. Here $\Ghat$ is estimated using 10,000 randomly selected pairs  $(X_{t}, X_{t+s})$ from the training data. While for iteration generation (denoted as ``iter GTS''), we let
\bes
\Xhat_{T+t_{new}}^{(j)}(s)=\ghat^{(s)}\left(X_{T+t_{new}-s}, [\eta^{(j)}_{T+t_{new}-s},\ldots,\eta^{(j)}_{T+t_{new}-1}]\right),  
\ees
where $\ghat^{(s)}$ is the $s$-times composition of the function $\ghat$. 

The $\Ghat$ for s-step generation (and $\ghat$ for iterative generation) are estimated for KL divergence (i.e., $f(x)=x\log x$) using neural networks in the simulation. Specifically, in Case 1 and 2,
the input $X_t$ initially goes through two fully connected layers, each with ReLU activation functions. Afterward, it is combined with the random noise vector $\eta$. This combined input is then passed through a single fully connected layer. The discriminator has two separate processing branches that embed $X_t$ and $X_{t+1}$ into low-dimensional vectors, respectively. These vectors are concatenated and further processed to produce an output score. The details of network structure can be found in Table \ref{table:nn}. While for Case 3,  in the generator network, $X_{t}$, $X_{t-1}$, and $X_{t-2}$ are each processed independently using three separate fully-connected layers. Afterward, they are combined with a random noise vector and passed through another fully-connected layer before producing the output. The discriminator in this case follows the same process as in Cases 1 and 2.

Different from supervised learning, generative learning does not have universally applicable metrics for evaluating the quality of generated samples \citep{theis2015note,borji2022pros}. Assessing the visual quality of produced images often depends on expert domain knowledge. Meanwhile, the GANs literature has seen significant efforts in understanding and developing evaluation metrics for generative performance. Several quantitative metrics have been introduced, such as Inception Score \citep{salimans2016improved}, Frechet Inception Distance \citep{heusel2017gans}, and Maximum Mean Discrepancy \citep{binkowski2018demystifying}, among others.

\begin{table}
	\centering
	\caption{\label{table:nn} The architecture of generator and discriminator for case 1 and 2.}
	\setlength{\tabcolsep}{0.7mm}{
		\begin{tabular}{*{4}{c}}
			\hline
			&Layer & Type & 
			\\ 
			&1	& fully connected layer (in dims = 1024, out dims = 256) & \\
			&2 & ReLU& \\   
			generator		&3 & fully connected layer (in dims = 256, out dims = 128) \\
			&4 &  ReLU& \\ 
			&5 & concatenate with random vector $\eta$&\\  \bigskip
			& 6 & fully connected layer (in dims = 148, out dims = 1024)
			\\
			discriminator	& \multicolumn{2}{c}{$\left.\begin{matrix}
					X_{t} \xrightarrow{\text{fc+ReLU}} 64 \xrightarrow{\text{fc}}  x_{t} \in \bbR^{64}\\ 
					X_{t+1} \xrightarrow{\text{fc+ReLU }}  x_{t+1} \in \bbR^{64}
				\end{matrix}\right\} 
				\text{concatenate}(x_{t}, x_{t+1})
				\xrightarrow{\text{fc+ReLU}} 64 \xrightarrow{\text{fc }}  \text{score} $ }& \\
			\hline
	\end{tabular} }
\end{table}%

\begin{table}
	\centering
	\caption{\label{table:n1}The NRMSE of mean estimation in Study I under different settings. The best and second best results are marked by {\color{pinegreen}\bf green} and {\bf bold} respectively.}
	\setlength{\tabcolsep}{4.2mm}{
		\begin{tabular}{*{7}{c}}
			\hline
			$T$	&Cases	& Methods & $s = 1$ & $s = 2$  & $s = 3$ &
			\\ 
			&& OLS & $\textcolor{pinegreen}{\bf{0.013}}$ (0.002) &  $\textcolor{pinegreen}{\bf{0.017}}$  (0.003) & $\textcolor{pinegreen}{\bf{0.023}}$   (0.004) & \\
			&case 1& Naive Baseline & 1.563 (0.039) & 1.932 (0.043) & 1.408 (0.098) & \\
			&& iter GTS & $\bf{0.800}$ (0.038) & 0.891 (0.050) & 0.985 (0.064) & \\ \smallskip
			&& s-step GTS & $\bf{0.800}$  (0.038) & $\bf{0.843 }$ (0.047) & $\bf{0.966}$ (0.062) & \\
			&& OLS & 0.973 (0.060) & 1.000 (0.021) & 0.995 (0.008) & \\
			1000 &case 2& Naive Baseline & 1.462 (0.129) & 1.455 (0.152) & 1.373 (0.178) & \\
			&& iter GTS & $\textcolor{pinegreen}{\bf{0.606}}$  (0.053) & $\bf{0.664}$ (0.064) & $\bf{0.702}$ (0.075) & \\ \smallskip
			&& s-step GTS & $\textcolor{pinegreen}{\bf{0.606}}$  (0.053) &  $\textcolor{pinegreen}{\bf{0.609}}$   (0.057) &  $\textcolor{pinegreen}{\bf{0.615}}$  (0.061) & \\

			&& OLS & 0.608 (0.021) & 0.609 (0.021) & 0.609 (0.020) & \\
			&case 3& Naive Baseline & 0.845 (0.034) & 0.843 (0.034) & 0.844 (0.033) & \\
			&& iter GTS & $\textcolor{pinegreen}{\bf{0.598}}$  (0.020) & $\textcolor{pinegreen}{\bf{0.595}}$ (0.020) & $\textcolor{pinegreen}{\bf{0.598}}$ (0.020) & \\ \bigskip
			&& s-step GTS & $\textcolor{pinegreen}{\bf{0.598}}$  (0.020) & $\textcolor{pinegreen}{\bf{0.595}}$ (0.020) & $\bf{0.602}$ (0.020) & \\
			&& OLS & $\textcolor{pinegreen}{\bf{0.007}}$  (0.001) &  $\textcolor{pinegreen}{\bf{0.010}}$ (0.002) & $\textcolor{pinegreen}{\bf{0.013}}$  (0.002) & \\
			&case 1& Naive Baseline & 1.563 (0.040) & 1.932 (0.043) & 1.408 (0.099) & \\
			&& iter GTS & $\bf{0.615}$ (0.098) & 0.717 (0.163) & 0.851 (0.226) & \\ \smallskip
			&& s-step GTS & $\bf{0.615}$ (0.098) & $\bf{0.707}$ (0.167) & $\bf{0.749}$ (0.079) & \\
			
			&& OLS & 0.973 (0.060) & 1.001 (0.021) & 0.995 (0.007) & \\
			5000 & case 2& Naive Baseline & 1.468 (0.132) & 1.452 (0.157) & 1.377 (0.181) & \\
			&& iter GTS & $\textcolor{pinegreen}{\bf{0.470}}$  (0.046) & $\bf{0.532}$ (0.056) & $\bf{0.576}$ (0.066) & \\ \smallskip
			&& s-step GTS & $\textcolor{pinegreen}{\bf{0.470}}$ (0.046) & $\textcolor{pinegreen}{\bf{0.470}}$ (0.048) & $\textcolor{pinegreen}{\bf{0.500}}$ (0.054) & \\
			&& OLS & 0.608 (0.020) & 0.610 (0.021) & 0.609 (0.020) & \\
			&case 3& Naive Baseline & 0.845 (0.032) & 0.846 (0.032) & 0.850 (0.033) & \\
			&& iter GTS & $\textcolor{pinegreen}{\bf{0.578}}$  (0.020) & $\bf{0.574}$ (0.020) & $\bf{0.595}$ (0.020) & \\
			&& s-step GTS & $\textcolor{pinegreen}{\bf{0.578}}$ (0.020) & $\textcolor{pinegreen}{\bf{0.566}}$  (0.020) & $\textcolor{pinegreen}{\bf{0.592}}$  (0.021) & \\
			\hline
		\end{tabular}
	}
\end{table}
Nonetheless, achieving a consensus regarding the evaluation of generative models remains an unresolved issue. In our approach, we compute the mean of the generated samples and present the results as the normalized root mean squared error (NRMSE) of the mean estimation. Specifically, for a given step $s$, let 
$\Xtil_{T+t_{new}}(s)=(1/J)\sum_{j=1}^J\Xtil_{T+t_{new}}^{(j)}(s)$ and $\Xhat_{T+t_{new}}(s)=(1/J)\sum_{j=1}^J\Xhat_{T+t_{new}}^{(j)}(s)$ be the estimated mean of the s-step and iterative generated samples, respectively. The NRMSE of the iterative generation is defined as
\bes
\text{NRMSE}(s) = \|\hX_{T+t_{new}}(s)  - \bbE X_{T+t_{new}}\|_F / \|\bbE X_{T+t_{new}}\|_F,
\ees
where  $\|\cdot\|_{F}$ denotes the Frobenius norm. The NRMSE of the s-step generation can be defined similarly.

The performance of our iterative and s-step generation are compared with two benchmark approaches. We first consider a naive baseline in which the prediction for $X_{T+t_{new}}$ is taken from the observation $s$-steps ahead for a  given $s=1,\ldots, S$, meaning that $\Xhat_{T+t_{new}}=X_{T+t_{new}-s}$. In addition, we consider a simple linear estimator obtained using Ordinary Least Squares (OLS). Specifically, the linear coefficients are estimated with a correctly specified order of lag, i.e., $\widehat{\phi}_1 = \argmin_{\phi} \sum_{t = 1}^{T}\| X_{t+1} - \phi X_{t}\|_F$ for Cases 1 and 2, and $(\widehat{\phi}_1, \widehat{\phi}_2, \widehat{\phi}_3) = \argmin_{\phi_1, \phi_2, \phi_3} \sum_{t = 1}^{T}\| X_{t+1} - \phi_1 X_{t} - \phi_2 X_{t-1} - \phi_3 X_{t-2}\|_F$ for Case 3. Clearly, in Case 1, OLS is a suitable choice as the model is linear. However, for Cases 2 and 3, OLS is mis-specified. 

We repeat the simulation for 100 times and report in Table \ref{table:n1}  the mean and standard deviation of the NRMSE for $s=1,2,3$ using different approaches. It is clear that our s-step and iterative generations exhibit competitive performance across almost all settings, particularly in the nonlinear Cases 2 and 3. In Case 1, the original model is linear, and as expected, OLS achieves the minimum NRMSE. Moreover, as $s$ increases, the problem becomes more challenging. However, both s-step and iterative generation maintain robust performance across different $s$. It is important to note that when $s=1$, the iterative and s-step generation methods are equivalent. For $s=2,3$, the s-step generation generally exhibits a slightly smaller NRMSE, though the difference is not significant.



\subsection{Study II: Multiple Time Series}\label{sim2}
\noindent
In this subsection, we consider a panel data setting with multiple time series. We consider two different sample sizes, $n=200, 500$ while fixing the time series length $T=20$.  We set the horizon $S=3$ as before, i.e., generate images in a maximum of $S=3$ steps for each subject. 

\begin{table}
	\centering
	\caption{\label{table:multi_n}The NRMSE of mean estimation in Study II under different settings. The best and second best results are marked by {\color{pinegreen}\bf green} and {\bf bold} respectively.}
	\setlength{\tabcolsep}{4.2mm}{
		\begin{tabular}{*{7}{c}}
			\hline
			$(n,T)$	&Cases	& Methods & $s = 1$ & $s = 2$  & $s = 3$ &
			\\ 
			&& OLS & $\textcolor{pinegreen}{\bf{0.037}}$  (0.001) & $\textcolor{pinegreen}{\bf{0.070}}$ (0.002) & $\textcolor{pinegreen}{\bf{0.096}}$  (0.003) & \\
			&case 1& Naive Baseline & 1.563 (0.039) & 1.925 (0.043) & 1.380 (0.097) & \\
			&& iter GTS & $\bf{0.729 }$ (0.036) & 0.813 (0.046) & 0.893 (0.056) & \\ \smallskip
			&& s-step GTS & $\bf{0.729 }$ (0.036) & $\bf{0.780}$ (0.044) & $\bf{0.885}$ (0.056) & \\
			&& OLS & 0.985 (0.029) & 1.000 (0.005) & 0.999 (0.001) & \\
			$(200,20)$&case 2& Naive Baseline & 1.459 (0.107) & 1.478 (0.127) & 1.415 (0.156) & \\
			&& iter GTS & $\textcolor{pinegreen}{\bf{0.583}}$ (0.049) & $\textcolor{pinegreen}{\bf{0.647}}$ (0.060) & $\textcolor{pinegreen}{\bf{0.690}}$  (0.071) & \\ \smallskip
			&& s-step GTS & $\textcolor{pinegreen}{\bf{0.583}}$ (0.049) & $\bf{0.689 }$ (0.068) &  $\bf{0.710}$ (0.077) & \\
			&& OLS & 0.617 (0.018) & 0.620 (0.018) & 0.625 (0.015) & \\
			&case 3& Naive Baseline & 0.847 (0.032) & 0.845 (0.033) & 0.846 (0.034) & \\
			&& iter GTS &  $\textcolor{pinegreen}{\bf{0.593}}$ (0.019) & $\textcolor{pinegreen}{\bf{0.590}}$ (0.020) & $\textcolor{pinegreen}{\bf{0.594}}$ (0.020) & \\ \bigskip
			&& s-step GTS & $\textcolor{pinegreen}{\bf{0.593}}$ (0.019) & $\bf{0.592}$ (0.020) & $\bf{0.597}$ (0.020) & \\
			
			&& OLS &   $\textcolor{pinegreen}{\bf{0.037}}$  (0.001) & $\textcolor{pinegreen}{\bf{0.069}}$ (0.002) & $\textcolor{pinegreen}{\bf{0.096}}$ (0.003) & \\
			&case 1& Naive Baseline & 1.564 (0.039) & 1.925 (0.043) & 1.379 (0.098) & \\
			&& iter GTS & $\bf{0.611}$ (0.040) & 0.667 (0.053) & $\bf{0.755}$ (0.072) & \\ \smallskip
			&& s-step GTS & $\bf{0.611}$ (0.040) &$\bf{0.647}$ (0.066) & 0.763 (0.056) & \\
			
			&& OLS & 0.985 (0.029) & 1.000 (0.005) & 0.999 (0.001) & \\
			$(500,20)$&case 2& Naive Baseline & 1.458 (0.109) & 1.478 (0.128) & 1.414 (0.156) & \\
			&& iter GTS & $\textcolor{pinegreen}{\bf{0.511}}$ (0.050) & $\textcolor{pinegreen}{\bf{0.578}}$  (0.060) & $\textcolor{pinegreen}{\bf{0.625}}$ (0.070) & \\ \smallskip
			&& s-step GTS &$\textcolor{pinegreen}{\bf{0.511}}$  (0.050) & $\bf{0.639}$ (0.069) & $\bf{0.682}$ (0.080) & \\
			&& OLS & 0.617 (0.018) & 0.619 (0.018) & 0.625 (0.015) & \\
			&case 3& Naive Baseline & 0.847 (0.032) & 0.845 (0.033) & 0.846 (0.034) & \\
			&& iter GTS & $\textcolor{pinegreen}{\bf{0.579}}$ (0.019) &$\textcolor{pinegreen}{\bf{0.577}}$ (0.020) & $\textcolor{pinegreen}{\bf{0.592}}$ (0.020) & \\
			&& s-step GTS & $\textcolor{pinegreen}{\bf{0.579}}$ (0.019) & $\bf{0.579}$ (0.020) & $\bf{0.593}$ (0.021) & \\
			\hline
		\end{tabular}
	}
\end{table}

Analogous to the single time series,  we repeat the simulation for 100 times and report in Table \ref{table:multi_n}  the mean and standard deviation of the NRMSE for $s=1,2,3$ using different approaches. Table \ref{table:multi_n}  shows a similar pattern as Table \ref{table:n1}. Both iterative and s-step generated images achieve the minimum NRMSE in Case 2 and 3. While under the linear Case 1, the OLS continues to achieve the lowest NRMSE. One notable difference in Table \ref{table:multi_n} is that the iterative generation achieves a lower NRMSE compared to the s-step generation in this case. One potential explanation for this is the increase in sample size. In comparison to the single time series setting, a sample size of $n=200$ or $500$ allows the iterative approach to obtain a better $g$ estimation, which in turn enhances the image generation process.

\section{The ADNI study}\label{sec:real}
\noindent
Driven by the goal of understanding the brain aging process, we in this section study real brain MRI data from the Alzheimer's Disease Neuroimaging Initiative (ADNI). 
Introduced by \cite{petersen2010alzheimer}, ADNI aims to investigate the progression of Alzheimer's disease (AD) by collecting sequential MRI scans of subjects classified as cognitively normal (CN), mildly cognitively impaired (MCI), and those with AD. 

Our study focuses on analyzing the brain's progression during the MCI stage. We  include a total of 565 participants from the MCI group in our analysis, each having a sequence of T1-weighted MRI scans with lengths varying between 3 and 9. We approach this as an imaging time series generation problem with multiple samples. However, a significant challenge with the dataset is the short length of each sample ($T_i\le 9$), a common issue in brain imaging analysis. 
Given the short length of these samples, we do not generate new images beyond a specific point $T$, as it would be difficult to evaluate their quality. Instead, we divide the dataset into a training set consisting of 450 samples and a testing set comprising 115 samples. We then train the generator $G$ using the training set and generate image sequences for the testing set given the starting point $X_0$. 
As in the simulation study, we also consider a naive baseline in which the prediction for $X_{i,s}$ is taken from the observation $s$-steps ahead, meaning that $\Xhat_{i, s}=X_{i, 0}$. It is worth mentioning that OLS is not suitable for MRI analysis and, as such, is not included in this context.


In our study, all brain T1-weighted MRI scans are processed through a standard pipeline which begins with a spatial adaptive non-local means (SANLM) denoising filter \citep{manjon2010adaptive}, then followed by resampling, bias-correction, affine-registration and unified segmentation \citep{ashburner2005unified}, skull-stripping and cerebellum removing. Each MRI is then locally intensity corrected and spatially normalized into the Montreal Neurological Institute (MNI) atlas space \citep{ashburner2007fast}.
These procedures result in processed images of size $169\times 205 \times 169$.
We further rescale the intensities of the resulting images to a range of $[-1,1]$. We select the central axial slice from each MRI and crop it to a size of $144 \times 192$ by removing the zero-valued voxels.

Figure \ref{fig:real-sequence} plots the original, iterative and s-step generated MRI sequence along with their difference to $X_0$ for one subject in the test set.  As shown by the original images, the brain changes gradually as age increases. This pattern is clearly captured by both iterative and s-step generations. 
To further assess the generated MRI images, we plot the starting image $X_0$, true image after s step (i.e., $X_{s}$), iteratively and $s$-step generated images (i.e., $\Xhat_{s}$ and $\Xtil_{s}$) for three subjects in Figure \ref{fig:real-subjects}.

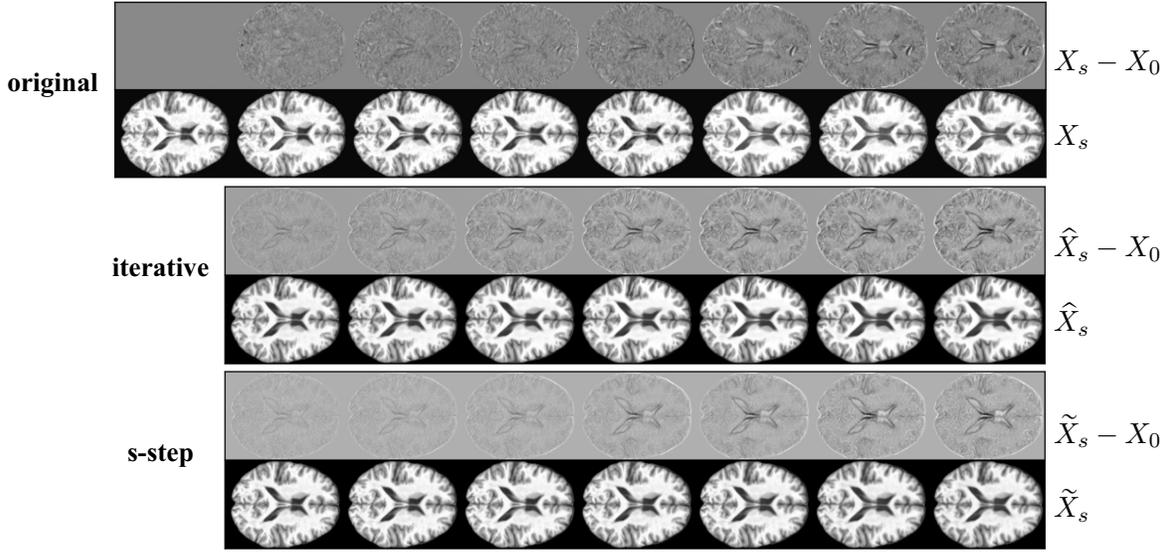
\begin{figure}
	\begin{overpic}[width=1\textwidth]{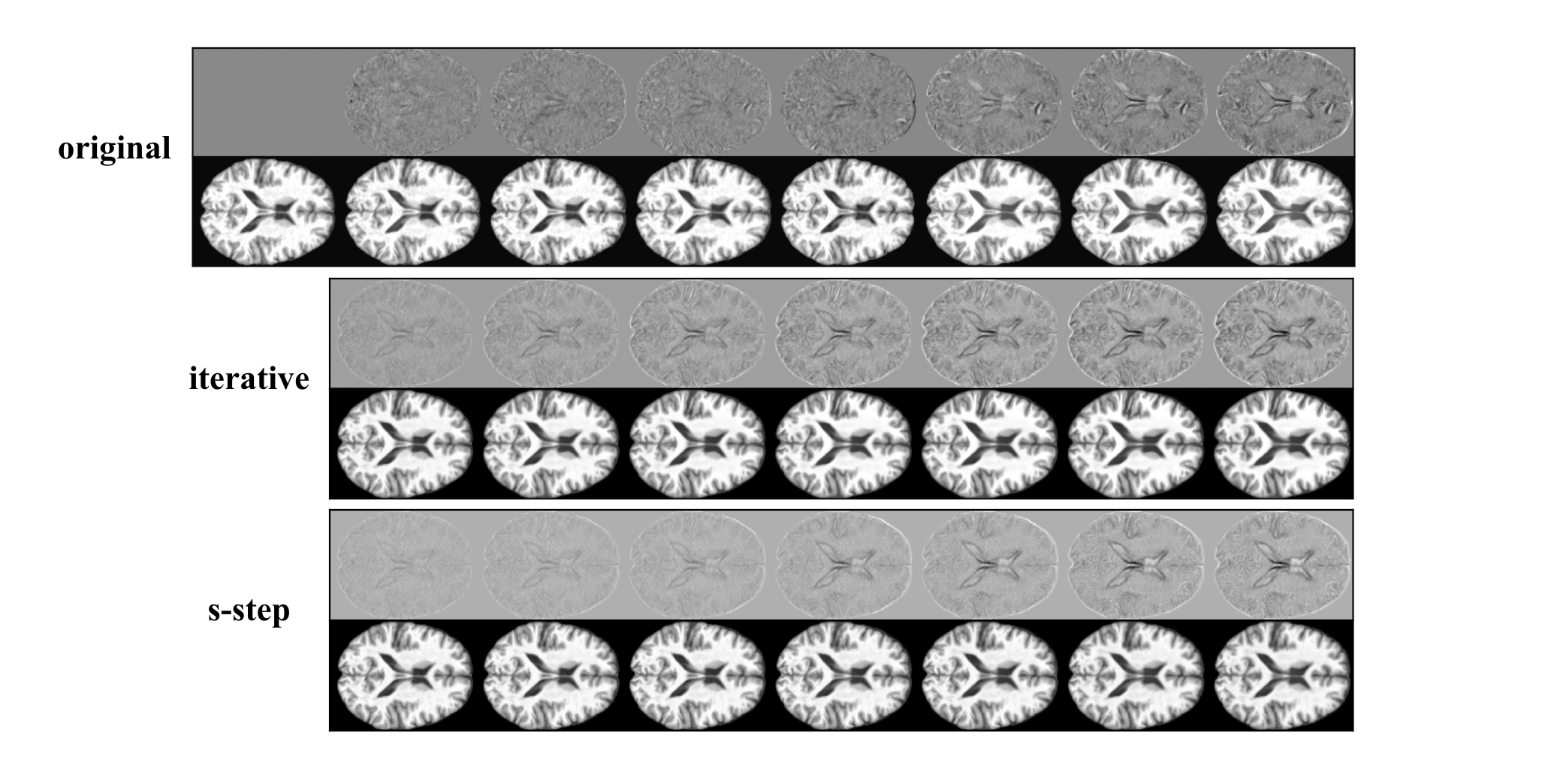}
		\put (88,42){$X_s - X_0$}
		\put (88,36){$X_s$}
		
		\put (88,27){$\hX_s - X_0$}	
		\put (88,21){$\hX_s$}
		
		\put (88,12){$\Xtil_s - X_0$}
		\put (88,6){$\Xtil_s$}
		
	\end{overpic}
	\centering
	\caption{The original, iterative and s-step generated MRI sequence $s=0,1,\ldots, 7$ along with their difference to $X_0$ for one subject.} 
\label{fig:real-sequence}
\end{figure}

Although the differences are subtle, we can observe that the image $X_{s}$ and the generated images, $\Xhat_{s}$ and $\Xtil_{s}$, are quite similar, but they all deviate from $X_0$ in several crucial regions. We highlight four of these regions in Figure \ref{fig:real-subjects}, including: a) cortical sulci, b) ventricles, c) edge of ventricles, and d) anterior inter-hemispheric fissure. 
More specifically, we can observe (from subjects 1, 2, and 3, region a) that the cortical sulci widen as age increases. The widening of cortical sulci may be associated with white matter degradation \citep{drayer1988imaging,walhovd2005effects}. This phenomenon is also observed in patients with Alzheimer's Disease \citep{migliaccio2012white}.
Additionally, the brain ventricles expand from time $0$ to $s$ as suggested by subjects 1 and 3, region b. The enlargement of ventricles during the aging process is one of the most striking features in structural brain scans across the lifespan \citep{macdonald2021mri}. Moreover, we notice that the edge of the ventricles becomes softer (darker region of subject 1, region c), and there is an increased presence of low signal areas adjacent to the ventricles (subject 2, regions b and c). From a clinical perspective, this observation suggests the existence of periventricular interstitial edema, which is linked to reduced ependyma activity and brain white matter atrophy \citep{todd2018ventricular}. Lastly, the anterior interhemispheric fissure deepens with aging, as demonstrated in subject 1, region d. In conclusion, the generated samples can potentially aid clinical analyses in identifying age-related brain issues.

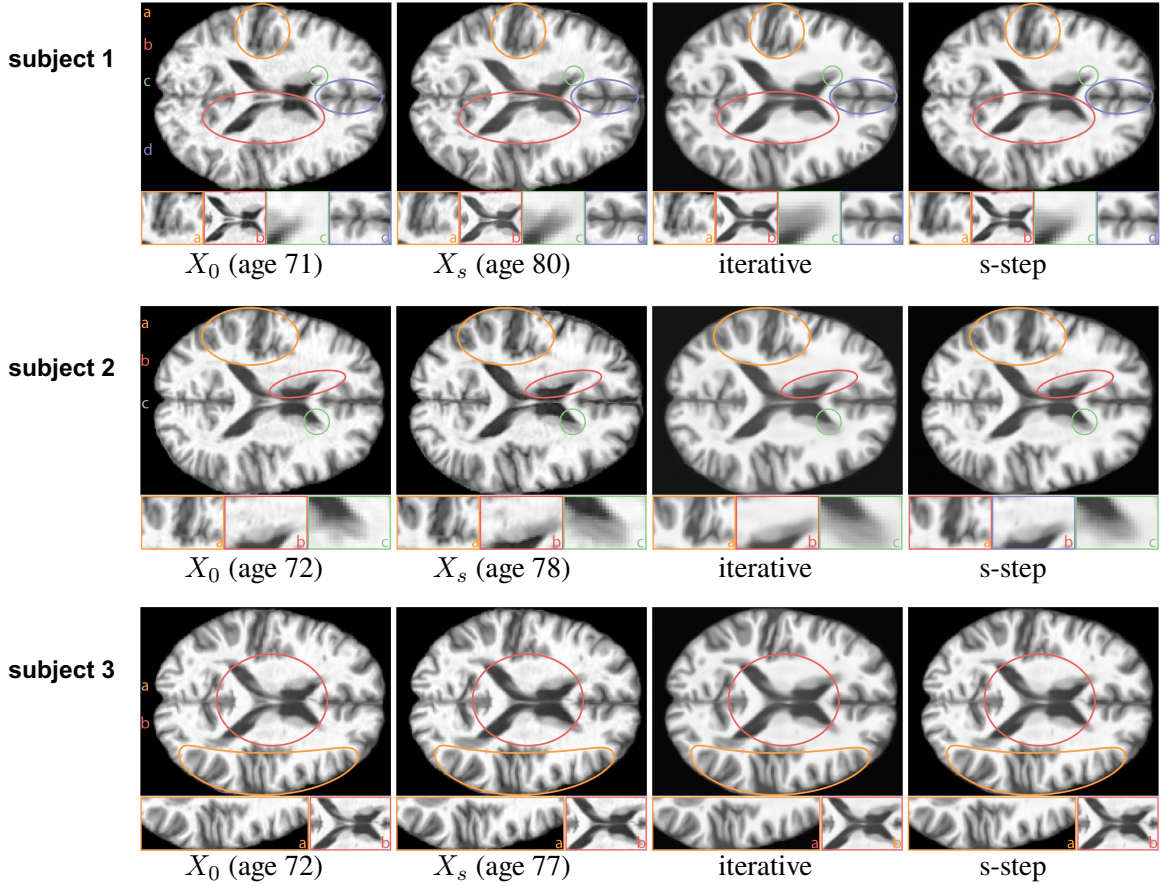
\begin{figure}
\centering
\begin{overpic}[width=1\textwidth]{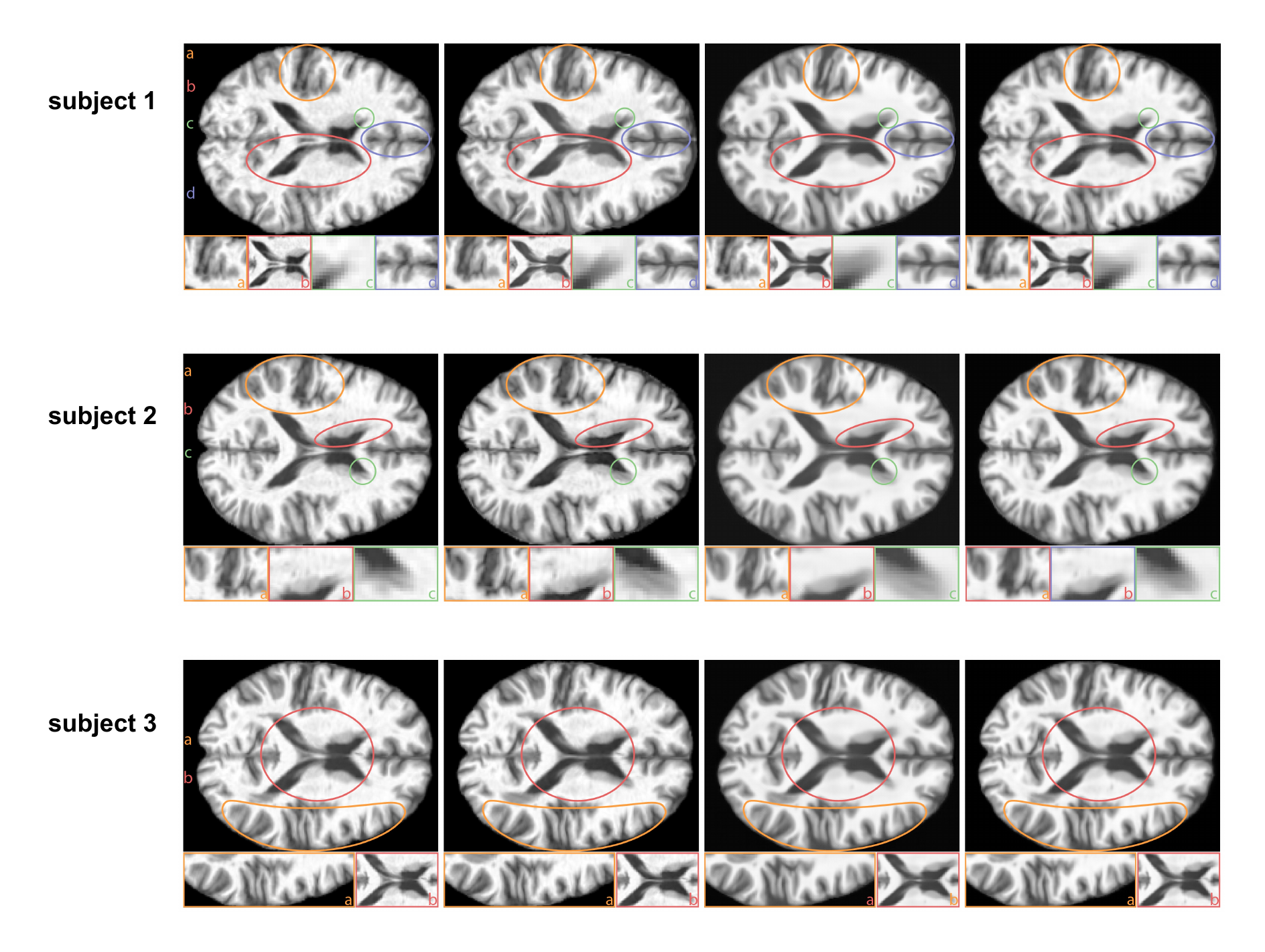}
	\put (18,50){$X_0$ (age 71)}	
	\put (38,50){$X_{s}$ (age 80)}
	\put (61,50){iterative}
	\put (82,50){s-step}
	\put (18,25.5){$X_0$ (age 72)}	
	\put (38,25.5){$X_{s}$ (age 78)}
	\put (61,25.5){iterative}
	\put (82,25.5){s-step}
	\put (18,1.5){$X_0$ (age 72)}	
	\put (38,1.5){$X_{s}$ (age 77)}
	\put (61,1.5){iterative}
	\put (82,1.5){s-step}
\end{overpic}
\caption{
	An illustration of the generated samples for three subjects in the testing group. For each subject, we present from left to right: starting image $X_{0}$, target $X_{s}$, and two generated samples $\hX_{s}$ and $\Xtil_{s}$ by iterative and s-step generation respectively. Moreover, we highlight (with different colors) four different regions that $X_{s}$ (along with $\Xhat_{s}$ and $\Xtil_{s}$) are most different from $X_0$, including:   {\bf {\color{MyYellow}{a)}} cortical sulci, {\color{MyRed}{b)}} ventricles, {\color{MyGreen} {c)}}
		edge of ventricles, {\color{MyPurple}d)} anterior interhemispheric fissure}.}
\label{fig:real-subjects}
\end{figure}

As discussed before, there is a lack of universally application metric for evaluating the quality of generated images. In this study, we consider three metrics to measure the difference between the target $X_{s}$ and our generations $\hX_{s}$: structural similarity index measure (SSIM), peak signal-to-noise ratio (PSNR), along with the previously introduced NRMSE. 
More specifically, SSIM calculates the similarity score between two images by comparing their luminance, contrast, and structure. Given two images $X$ and $Y$, it is defined as
\bes
\text{SSIM} (X,Y) = \frac{(2\mu_x \mu_y + c_1)(2\sigma_{xy}+c_2)}{(\mu_x^2+\mu_y^2+c_1)(\sigma_x^2+\sigma_y^2+c_2)},
\ees
where $(\mu_x,\mu_y)$, $(\sigma_x^2,\sigma_y^2)$ are the mean and variance of pixel values in $X$ and $Y$ respectively.  The $\sigma_{xy}$ is the covariance of $X$ and $Y$ and $c_1$, $c_2$ are constants, to be specified in the supplementary material.	
PSNR is a widely used engineering term for measuring the reconstruction quality for images subject to lossy compression. It is typically defined using the mean squared error:
\bes
\text{PSNR}(X,Y) = 10 \log_{10}\left( c^2/\text{MSE}(X,Y)\right),
\ees
where $c$ is the maximum pixel value among $X,Y$ and $\text{MSE}(X,Y) =  \|X -Y\|_F^2/(D_1D_2)$.
Clearly, better image generation performance is indicated by higher values of SSIM and PSNR, as well as smaller values of NRMSE.

In Figure \ref{fig:real-table}, we present for each $s$ the mean SSIM, PSNR, and NRMSE over all subjects in the testing set.
The results clearly indicate that both s-step and iterative generations outperform the benchmark in all three metrics across almost all the $s=1,\ldots, 9$, providing strong evidence of the high quality of our generated images. As $s$ increases, the generation problem becomes more challenging. Consequently, we observe a decrease in both SSIM and PSNR for both iterative and s-step generations, indicating a decline in generation quality. On the other hand, the NRMSE for both approaches increase, suggesting a higher level of error in the generated images. When further comparing the iterative and s-step generation, the s-step generation shows a dominating performance in this study. The major reason is due to the limited sample size. With fewer than 500 subjects in training and the observed sequence for each subject being less than 9, a direct s-step generation approach proves advantageous compared to the iterative generation method.

\begin{figure}
\centering
\includegraphics[width=1\columnwidth]{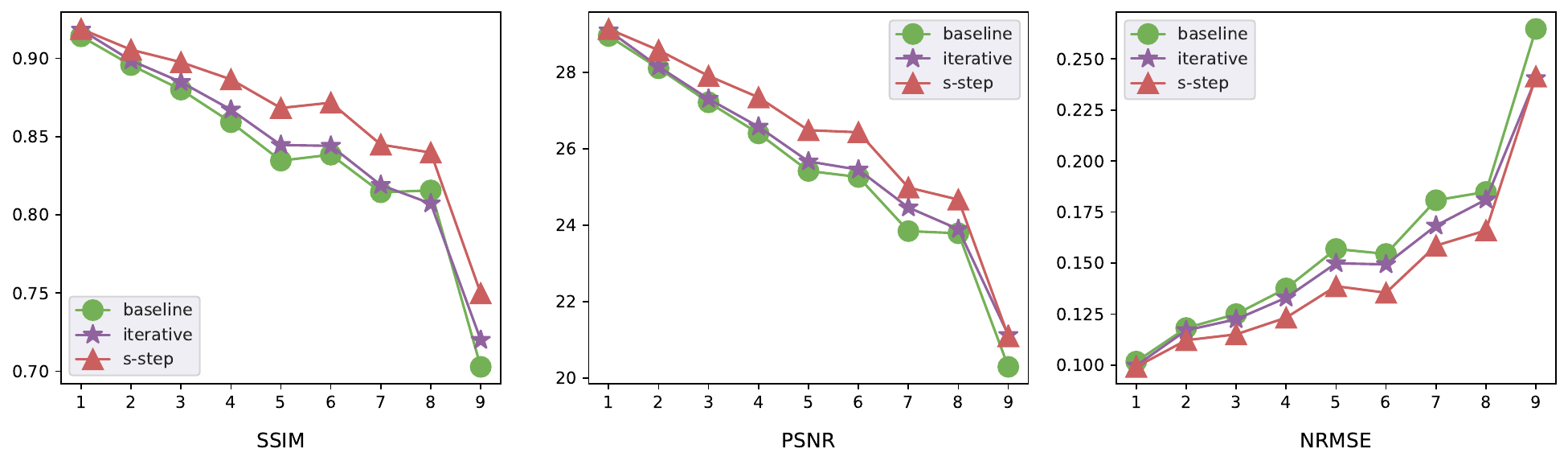}

\caption{The mean SSIM, PSNR, and NRMSE over all testing subjects in the ADNI data analysis for $s=1,\ldots, 9$.
}
\label{fig:real-table}
\end{figure}

In summary, this study further validates the effectiveness of our approach in generating image sequences. Incorporating the generated image sequences as data augmentation \citep{chen2022generative} could further enhance the performance of downstream tasks, such as Alzheimer's disease detection \citep{xia2022adversarial}. As mentioned before, AD is the most prevalent neurodegenerative disorder, progressively leading to irreversible neuronal damage.
The early diagnosis of AD and its syndromes, such as MCI, is of significant importance. We believe that the proposed image time series learning offers valuable assistance in understanding and identifying AD, as well as other aging-related brain issues.


\newpage
\bibliography{reference}
\appendix
\appendixpage
\addappheadtotoc
\noindent In the supplementary material, we provide  proofs in the following order: Theorem \ref{thm0}, Theorem \ref{thm1}, Proposition \ref{prop3}, Example \ref{prop4new}, Proposition \ref{prop6}, Proposition \ref{prop7}, Proposition \ref{prop3-1}, Theorem \ref{thm2}, Theorem \ref{mainthm}, Theorem  \ref{thm14}, Proposition \ref{prop: delta3}. Moreover, we include three additional lemmas. Furthermore, we present more details on the numerical implementations. 
\section{Proofs}\label{appendix:proofs}
\subsection{Proof of Theorem \ref{thm0}}
\begin{proof}
	By Lemma \ref{lm1}, we can find a measurable function $g$ and random variable $\Xhat_1^0$ such that $\Xhat_1^0 = g(\eta_0,\Xhat_0^0)$ and $(X_0,X_1)\overset{\text{d}}{=} (\Xhat_0^0,\Xhat_1^0)$. Then define $\Xhat_t^0 = g(\eta_{t-1}, \Xhat_{t-1}^0)$ for $t \ge 2$, we will use mathematical induction to prove that $(X_0,X_1,\cdots,X_s) \overset{\text{d}}{=} (\Xhat_0^0,\Xhat_1^0,\cdots,\Xhat_s^0)$ for $s\ge 1$.\\
	Because the construction of $g$ makes $(X_0,X_1)\overset{\text{d}}{=} (\Xhat_0^0,\Xhat_1^0)$, the conclusion is automatically held when $n=1$.\\
	For $s>1$, suppose $(X_0,X_1,\cdots,X_{s-1}) \overset{\text{d}}{=} (\Xhat_0^0,\Xhat_1^0,\cdots,\Xhat_{s-1}^0)$, we will prove that $(X_0,X_1,\cdots,X_s) \overset{\text{d}}{=} (\Xhat_0^0,\Xhat_1^0,\cdots,\Xhat_s^0)$ below.\\
	Note that $(X_0,X_1)\overset{\text{d}}{=} (\Xhat_0^0,\Xhat_1^0)$, then $\Xhat_1^0\vert \Xhat_0^0 = x\overset{\text{d}}{=} X_1\vert X_0 = x$. Which means $X_1\vert X_0 = x \overset{\text{d}}{=} g(\eta_0, x)$. \\
	Because $\eta_{s-1}$ is independent of $\Xhat_0^0$ and $\left\{\eta_i\right\}_{i=0}^{s-2}$, $\eta_{s-1}$ is independent of $\left\{\Xhat_i^0\right\}_{i=0}^{s-2}$, then:
	$$
	\Xhat_s^0\vert \Xhat_{s-1}^0=x_{s-1},\cdots,\Xhat_0^0=x_0 = g(\eta_{s-1}, x_{s-1})
	$$
	Because $\eta_{s-1} \overset{\text{d}}{=} \eta_0$, combining the given conditions, we have:
	\begin{align*}
		g(\eta_{s-1}, x_{s-1}) &\overset{\text{d}}{=} g(\eta_0, x_{s-1}) \\
		&= \Xhat_1^0 \vert \Xhat_0^0 = x_{s-1} \\
		&\overset{\text{d}}{=} X_1\vert X_0 = x_{s-1} \\
		&\overset{\text{d}}{=} X_s\vert X_{s-1} = x_{s-1} \\
		&\overset{\text{d}}{=} X_s\vert X_{s-1} =x_{s-1},\cdots,X_0=x_0
	\end{align*}
	which shows that: 
	$$\Xhat_s^0\vert \Xhat_{s-1}^0=x_{s-1},\cdots,\Xhat_0^0=x_0 \overset{\text{d}}{=} X_s\vert X_{s-1} =x_{s-1},\cdots,X_0=x_0
	$$
	Combining $(X_0,X_1,\cdots,X_{s-1}) \overset{\text{d}}{=} (\Xhat_0^0,\Xhat_1^0,\cdots,\Xhat_{s-1}^0)$, therefore: 
	$$
	(X_0,X_1,\cdots,X_s) \overset{\text{d}}{=} (\Xhat_0^0,\Xhat_1^0,\cdots,\Xhat_s^0)
	$$
\end{proof}

\subsection{Proof of Theorem \ref{thm1}}
\begin{proof}
	We define the measurable function $G$ such that:
	\bes
	G(\cdot,\cdot,1) = g(\cdot,\cdot)
	\ees
	For $t\ge 2$, we apply Lemma \ref{lm1} to $X_0,X_t$ and $\eta_{t}$, there exist a measurable function $G_t$ such that:
	\bes
	X_t \overset{\text{d}}{=}G_t(\eta_{t}, X_0)
	\ees
	Let $G(\cdot,\cdot, t) = G_t(\cdot,\cdot)$, $\Xtil_t^0 = G(\eta_t, X_0, t)$, then:
	\bes
	\Xtil_t^0 \overset{\text{d}}{=} X_t
	\ees
	Therefore, such function $G$ satisfies the condition.
	%
\end{proof}


\subsection{Proof of Proposition \ref{prop3}}
\begin{proof}

	For $G\in\mathcal{G}, H\in\mathcal{H}$ and $1\le s\le S$, define:
	\begin{align*}
		&G_s(\cdot,\cdot) := G(\cdot,\cdot,s) \\
		&H_s(\cdot,\cdot) := H(\cdot,\cdot,s)
	\end{align*}
	
	Then (\ref{LhatGD}) can be rewritten as:
	\bel{LGD1new}
	\tilde{\mathcal{L}}(G,H) = \frac{1} {|\Omega|}\sum\limits_{(t,s)\in \Omega}[H_s(X_{t},G_s(\eta_{t},X_{t})) - f^*(H_s(X_{t},X_{t+s}))]
	\eel
	For $1\le s\le S$, let:
	\bes
	\tilde{\mathcal{L}}_s(G_s,H_s) = \frac1{\vert\Omega\vert}\sum\limits_{t=0}^{T-s}[H_s(X_{t},G_s(\eta_{t},X_{t})) - f^*(H_s(X_{t},X_{t+s}))]
	\ees
	then,
	\bes
	\tilde{\mathcal{L}}(G,H) = \sum\limits_{s=1}^{S}\tilde{\mathcal{L}}_s(G_s,H_s)
	\ees
	According to the above equation, the optimization problem (\ref{LhatGD}) can be decomposed as the following $S$ optimization problems:
	\bel{op3}
	(\Ghat_s,\Hhat_s) = \arg\mathop{\min}_{G_s\in\mathcal{G}_s} \mathop{\max}_{H_s\in \mathcal{H}_s} \tilde{\mathcal{L}}_s(G_s,H_s) \ , \ 1\le s\le S
	\eel
	In particular, let $s = 1$ and note that $\tilde{\mathcal{L}}_1(G_1,H_1) =\frac{T}{\vert\Omega\vert}\widehat{\mathcal{L}}(G_1,H_1)$, we have:
	\bes
	(\Ghat_1,\Hhat_1) = (\ghat,\hhat)
	\ees
\end{proof}
\subsection{Proof of Example \ref{prop4new}}
\begin{proof}
	Obviously, for $t\ge 0$, $X_t$ follows gaussian distribution with mean $\nu_t = 0$. Suppose $\Sigma_t$ is the covariance matirx of $X_t$, then the covariance matrices sequence satisfies the following recurrence relationship:
	\bel{Sigmat}
	\Sigma_{t+1} = \phi_2 \Sigma_t \phi_2^T + \phi_1 \phi_1^T
	\eel
	In addition, $\Sigma_t$ converges to a symmetric matrix $\Sigma$, whose explicit expression is:
	\bes
	\Sigma =  \sum\limits_{t=0}^\infty \phi_2^t \phi_1 \phi_1^T (\phi_2^T)^t
	\ees
	In (\ref{Sigmat}), let $t\to\infty$, we have:
	\bes
	\Sigma = \phi_2 \Sigma \phi_2^T + \phi_1 \phi_1^T
	\ees
	then,
	\bes
	\Sigma_{t+1}-\Sigma_t = \phi_2(\Sigma_t - \Sigma) \phi_2^T
	\ees
	By iteration:
	\bes
	\Sigma_t - \Sigma = \phi_2^t(\Sigma_0- \Sigma)(\phi_2^T)^t
	\ees
	
	Let $\sigma = \sigma_{\text{max}}(\phi_2)$. Because $\phi_2$ is symmetric, there exists an orthogonal matrix $Q= (q_{ij})_{p\times p}$, such that:
	\bes
	\phi_2 = Q U Q^T
	\ees
	where $U = \diag\left\{u_1, u_2,\cdots, u_p\right\}$ with $\vert u_i\vert \le \sigma$, $1\le i \le p$.
	
	Denote $M_0 = Q^T(\Sigma_0-\Sigma)Q$, then:
	\bes
	\Sigma_t - \Sigma = Q^T U^t M_0 U^t Q
	\ees
	let $M_0 = (m_{ij})_{p\times p}$, then:
	\bel{St-S}
	\vert (\Sigma_t - \Sigma)_{ij}\vert = \vert\sum\limits_{k,l=1}^p q_{ki}q_{lj} m_{kl}u_k^t u_l^t \vert \le K \sigma^{2t}
	\eel
	where $K = p^2 \max\left\{\vert q_{ij}\vert\right\}^2 \cdot \max\left\{\vert m_{ij}\vert\right\}$ is a constant.
	
	Let $p(x) = (2\pi)^{-p/2} (\vert\Sigma\vert)^{-1/2} \exp\left\{-\frac{1}{2}x^T \Sigma^{-1} x\right\}$ be the density function of $N(0, \Sigma)$, then:
	\begin{align*}
		&\vert p_{X_t}(x)-p(x)\vert \\
		=& (2\pi)^{-p/2} \vert\Sigma_t\vert^{-1/2} \exp\left\{-\frac{1}{2}x^T \Sigma_t^{-1} x\right\} - (2\pi)^{-p/2} \vert\Sigma\vert^{-1/2} \exp\left\{-\frac{1}{2}x^T \Sigma^{-1} x\right\}\vert \\
		\le& (2\pi)^{-p/2} \vert \vert \Sigma_t\vert^{-1/2} - \vert \Sigma\vert^{-1/2}\vert \exp\left\{-\frac{1}{2}x^T \Sigma_t^{-1} x\right\} \\&+
		(2\pi)^{-p/2} (\vert\Sigma\vert)^{-1/2} \vert \exp\left\{-\frac{1}{2}x^T \Sigma_t^{-1} x\right\}- \exp\left\{-\frac{1}{2}x^T \Sigma^{-1} x\right\}\vert
	\end{align*}
	
	From (\ref{St-S}), there exists a constant $K_1$, such that:
	\bes
	\vert \vert\Sigma_t\vert-\vert \Sigma\vert\vert = \vert \vert \Sigma + (\Sigma_t-\Sigma)\vert - \vert \Sigma\vert \vert \le K_1 \sigma^{2t}
	\ees
	then:
	\begin{align*}
		&\int (2\pi)^{-p/2} \vert \vert \Sigma_t\vert^{-1/2} - \vert \Sigma\vert^{-1/2}\vert \exp\left\{-\frac{1}{2}x^T \Sigma_t^{-1} x\right\} dx\\
		=& \vert \vert \Sigma_t\vert^{-1/2} - \vert \Sigma\vert^{-1/2}\vert \cdot \vert \Sigma_t \vert^{1/2} \\
		=&  \vert\vert \Sigma_t\vert^{1/2} - \vert \Sigma\vert^{1/2}\vert \cdot \vert \Sigma \vert^{-1/2} \\
		\le&  \vert\vert\Sigma_t\vert^{1/2} - \vert \Sigma\vert^{1/2}\vert\cdot  \vert\vert\Sigma_t\vert^{1/2} + \vert \Sigma\vert^{1/2}\vert\cdot \vert \Sigma \vert^{-1}\\
		=&\vert \vert\Sigma_t\vert-\vert \Sigma\vert\vert \cdot \vert \Sigma \vert^{-1} \\
		\le& K_1 \sigma^{2t}\vert \Sigma \vert^{-1}
	\end{align*}
	For $x\in \mathbb{R}^p$, if $x^T(\Sigma_t^{-1}-\Sigma^{-1})x \ge 0$, then:
	\bes
	&&\vert \exp(-\frac{1}{2}x^T \Sigma_t^{-1} x)- \exp\left\{-\frac{1}{2}x^T \Sigma^{-1} x\right\}\vert \cr &=& \exp\left\{-\frac{1}{2}x^T \Sigma^{-1} x\right\}\cdot \vert 1 - \exp\left\{-\frac{1}{2}x^T(\Sigma_t^{-1}-\Sigma^{-1})x\right\}\vert 
	\cr &\le& \frac{1}{2}x^T(\Sigma_t^{-1}-\Sigma^{-1})x \cdot\exp\left\{-\frac{1}{2}x^T \Sigma^{-1} x\right\}
	\ees
	if $x^T(\Sigma_t^{-1}-\Sigma^{-1})x < 0$:
	\bes
	&&\vert \exp\left\{-\frac{1}{2}x^T \Sigma_t^{-1} x\right\}- \exp\left\{-\frac{1}{2}x^T \Sigma^{-1} x\right\}\vert 
	\cr &=& \exp\left\{-\frac{1}{2}x^T \Sigma_t^{-1} x\right\}\cdot \vert 1 - \exp\left\{-\frac{1}{2}x^T(\Sigma^{-1}-\Sigma_t^{-1})x\right\}\vert 
	\cr &\le& \frac{1}{2}x^T(\Sigma^{-1}-\Sigma_t^{-1})x \cdot \exp\left\{-\frac{1}{2}x^T \Sigma_t^{-1} x\right\}
	\ees
	therefore,
	\bes
	\vert \exp\left\{-\frac{1}{2}x^T \Sigma_t^{-1} x\right\}- \exp\left\{-\frac{1}{2}x^T \Sigma^{-1} x\right\}\vert &\le& \frac{1}{2}\vert x^T(\Sigma_t^{-1}-\Sigma^{-1})x\vert \cdot(\exp\left\{-\frac{1}{2}x^T \Sigma^{-1} x\right\}
	\cr &&\ \ +\exp\left\{-\frac{1}{2}x^T \Sigma_t^{-1} x\right\})
	\ees
	Since $\Sigma_t^{-1}-\Sigma^{-1} = \Sigma_t^{-1}(\Sigma-\Sigma_t)\Sigma^{-1}$, by (\ref{St-S}), there exists a constant $K_2$ such that:
	\bes
	(\Sigma_t^{-1}-\Sigma^{-1})_{ij} \le K_2 \sigma^{2t}
	\ees
	hence, for $x\in \mathbb{R}^p$,
	\begin{align*}
		\vert x^T(\Sigma_t^{-1}-\Sigma^{-1})x\vert &\le \sigma_{\text{max}}(\Sigma_t^{-1}-\Sigma^{-1})\cdot x^Tx \\
		&= \mathop{\sup}_{\vert z\vert=1}\vert (\Sigma_t^{-1}-\Sigma^{-1})z\vert_2 \cdot x^Tx \\
		&\le \sqrt{p}K_2 \sigma^{2t} x^Tx
	\end{align*}
	therefore:
	\begin{align*}
		&\int (2\pi)^{-p/2} \vert\Sigma\vert^{-1/2} \vert \exp\left\{-\frac{1}{2}x^T \Sigma_t^{-1} x\right\}- \exp\left\{-\frac{1}{2}x^T \Sigma^{-1} x\right\}\vert dx\\
		\le& \int (2\pi)^{-p/2} \vert\Sigma\vert^{-1/2} \frac{\sqrt{p}}{2}K_2 \sigma^{2t} x^Tx (\exp\left\{-\frac{1}{2}x^T \Sigma^{-1} x\right\}+\exp\left\{-\frac{1}{2}x^T \Sigma_t^{-1} x\right\})dx \\
		=&\int (2\pi)^{-p/2} \vert\Sigma\vert^{-1/2} \frac{\sqrt{p}}{2}K_2 \sigma^{2t} x^Tx \exp\left\{-\frac{1}{2}x^T \Sigma^{-1} x\right\}dx \\
		&+ \int (2\pi)^{-p/2} \vert\Sigma\vert^{-1/2} \frac{\sqrt{p}}{2}K_2 \sigma^{2t} x^Tx \exp\left\{-\frac{1}{2}x^T \Sigma_t^{-1} x\right\}dx \\
		=&\int (2\pi)^{-p/2} \vert\Sigma\vert^{-1/2}\frac{\sqrt{p}}{2}K_2 \sigma^{2t} z^T\Sigma z\exp\left\{-\frac{1}{2}z^Tz\right\}\vert\Sigma\vert^{1/2}dz\\
		&+ \int (2\pi)^{-p/2} \vert\Sigma\vert^{-1/2}\frac{\sqrt{p}}{2}K_2 \sigma^{2t} z^T\Sigma_t z\exp\left\{-\frac{1}{2}z^Tz\right\}\vert\Sigma_t\vert^{1/2}dz\\
		\le&\int (2\pi)^{-p/2} \vert\Sigma\vert^{-1/2}\frac{\sqrt{p}}{2}K_2 \sigma^{2t}(\vert\Sigma\vert^{1/2}\sigma_{\text{max}}(\Sigma) + \vert\Sigma_t\vert^{1/2}\sigma_{\text{max}}(\Sigma_t)) z^Tz \exp\left\{-\frac{1}{2}z^Tz\right\}dz\\
		=&p\cdot\vert\Sigma\vert^{-1/2}\frac{\sqrt{p}}{2}K_2 \sigma^{2t}(\vert\Sigma\vert^{1/2}\sigma_{\text{max}}(\Sigma) + \vert\Sigma_t\vert^{1/2}\sigma_{\text{max}}(\Sigma_t))\\
		\le&K_3 \sigma^{2t}
	\end{align*}
	where $K_3$ is a constant.
	
	Combining the above results, we have:
	\begin{align*}
		\int \vert p_{X_t}(x)-p(x)\vert dx \le K_1\sigma^{2t}\vert \Sigma\vert^{-1} + K_3 \sigma^{2t} = \mathcal{O}(\sigma^{2t})
	\end{align*}
	
\end{proof}

\subsection{Proof of Proposition \ref{prop6}}
\begin{proof}
	We recursively prove the following equation for $1\le s \le S$:
	\bel{Rec}
	\Vert p_{\Xhat_T,\cdots,\Xhat_{T+s}} - p_{X_T,\cdots,X_{T+s}}\Vert_{L_1} \le \int \sum\limits_{r=0}^{s-1} p_{T+r}(x)\cdot \Vert p(\cdot\vert x)-q(\cdot\vert x) \Vert_{L_1} dx
	\eel
	For $s = 1$, note that:
	\bes
	p_{\Xhat_T,\Xhat_{T+1}}(x,y) - p_{X_T,X_{T+1}}(x,y) = p_T(x)[q(y\vert x) - p(y\vert x)]
	\ees
	then:
	\bes
	\Vert p_{\Xhat_T,\Xhat_{T+1}} - p_{X_T,X_{T+1}}\Vert_{L_1} &=& \int \vert p_{\Xhat_T,\Xhat_{T+1}}(x,y) - p_{X_T,X_{T+1}}(x,y)\vert dxdy
	\cr & =&\int p_T(x) \Vert p(\cdot\vert x)-q(\cdot\vert x) \Vert_{L_1} dx
	\ees
	therefore, the equation (\ref{Rec}) holds for $s =1$.
	
	For $1<k\le S$, assuming that (\ref{Rec}) holds for $s = k - 1$, now we consider the case of $s = k$.
	Note that for $1\le s\le S$:
	\begin{align*}
		&p_{\Xhat_T,\cdots,\Xhat_{T+s}}(x_0, \cdots, x_s) = p_T(x_0)\cdot\prod\limits_{i=1}^{s} q(x_i\vert x_{i-1}) \\
		&p_{X_T,\cdots,X_{T+s}}(x_0, \cdots, x_s) = p_T(x_0)\cdot\prod\limits_{i=1}^{s} p(x_i\vert x_{i-1})
	\end{align*}
	then:
	\begin{align*}
		&\Vert p_{\Xhat_T,\cdots,\Xhat_{T+k}} - p_{X_T,\cdots,X_{T+k}}\Vert_{L_1} \\
		=& \int p_T(x_0)\vert \prod\limits_{i=1}^{k} q(x_i\vert x_{i-1}) - \prod\limits_{i=1}^{k} p(x_i\vert x_{i-1}) \vert d x_{0:k} \\
		\le& \int p_T(x_0)\cdot\vert \prod\limits_{i=1}^{k-1} q(x_i\vert x_{i-1}) - \prod\limits_{i=1}^{k-1} p(x_i\vert x_{i-1}) \vert \cdot q(x_k\vert x_{k-1}) d x_{0:k}\\
		&\ + \int p_T(x_0)\cdot\prod\limits_{i=1}^{k-1} p(x_i\vert x_{i-1})\cdot \vert q(x_k\vert x_{k-1}) - p(x_k\vert x_{k-1})\vert d x_{0:k} \\
		=& \int p_T(x_0)\cdot\vert \prod\limits_{i=1}^{k-1} q(x_i\vert x_{i-1}) - \prod\limits_{i=1}^{k-1} p(x_i\vert x_{i-1}) \vert d x_{0:k-1} \\
		&\ + \int p_{X_T,\cdots,X_{T+k-1}}(x_0,\cdots,x_{k-1}) \cdot \vert q(x_k\vert x_{k-1}) - p(x_k\vert x_{k-1})\vert d x_{0:k} \\
		=&\Vert p_{\Xhat_T,\cdots,\Xhat_{T+k-1}} - p_{X_T,\cdots,X_{T+k-1}}\Vert_{L_1} \\
		&\ + \int p_{T+k-1}(x_{k-1})\cdot \vert q(x_k\vert x_{k-1}) - p(x_k\vert x_{k-1})\vert d x_{k}dx_{k-1} \\
		=& \Vert p_{\Xtil_T,\cdots,\Xtil_{T+k-1}} - p_{X_T,\cdots,X_{T+k-1}}\Vert_{L_1} \\
		&\ + \int p_{T+k-1}(x_{k-1})\cdot \Vert p(\cdot\vert x_{k-1}) - q(\cdot\vert x_{k-1})\Vert_{L_1} dx_{k-1}
	\end{align*}	
	According to the hypothesis of induction,
	\bes
	\Vert p_{\Xhat_T,\cdots,\Xhat_{T+k-1}} - p_{X_T,\cdots,X_{T+k-1}}\Vert_{L_1} \le \int \sum\limits_{r=0}^{k-2} p_{T+r}(x)\cdot \Vert p(\cdot\vert x)-q(\cdot\vert x) \Vert_{L_1} dx
	\ees
	hence, we have:
	\bes
	&&\Vert p_{\Xhat_T,\cdots,\Xhat_{T+k}} - p_{X_T,\cdots,X_{T+k}}\Vert_{L_1} \cr&\le& \Vert p_{\Xhat_T,\cdots,\Xhat_{T+k-1}} - p_{X_T,\cdots,X_{T+k-1}}\Vert_{L_1} 
	\cr && \ \ \ + \int p_{T+k-1}(x_{k-1})\cdot \Vert p(\cdot\vert x_{k-1}) - q(\cdot\vert x_{k-1})\Vert_{L_1} dx_{k-1}
	\cr  &\le& \sum\limits_{r=0}^{k-2} p_{T+r}(x)\cdot \Vert p(\cdot\vert x)-q(\cdot\vert x) \Vert_{L_1} dx 
	\cr && \ \ \ + \int p_{T+k-1}(x_{k-1})\cdot \Vert p(\cdot\vert x_{k-1}) - q(\cdot\vert x_{k-1})\Vert_{L_1} dx_{k-1}
	\cr  &= &\sum\limits_{r=0}^{k-1} p_{T+r}(x)\cdot \Vert p(\cdot\vert x)-q(\cdot\vert x) \Vert_{L_1} dx
	\ees
	then (\ref{Rec}) holds for $s = k$.
	
	Therefore, the inequality (\ref{Rec}) holds for all $1\le s\le S$.
	In particular, taking $s = S$, by Assumption \ref{a2}, we have:
	\bes
	&&\Vert p_{\Xhat_T,\cdots,\Xhat_{T+S}} - p_{X_T,\cdots,X_{T+S}}\Vert_{L_1} 
	\cr&\le& \int \sum\limits_{r=0}^{S-1} p_{T+r}(x)\cdot \Vert p(\cdot\vert x)-q(\cdot\vert x) \Vert_{L_1} dx
	\cr &\le& S\int p_{T}(x)\cdot \Vert p(\cdot\vert x)-q(\cdot\vert x) \Vert_{L_1} dx + \int \sum\limits_{r=0}^{S-1} \vert p_{T+r}(x)-p_T(x)\vert \cdot \Vert p(\cdot\vert x)-q(\cdot\vert x) \Vert_{L_1} dx
	\cr &\le& S\int p_{T}(x)\cdot \Vert p(\cdot\vert x)-q(\cdot\vert x) \Vert_{L_1} dx+ 2\sum\limits_{r=0}^{S-1}\Vert p_{T+r}-p_T\Vert_{L_1}
	\cr &\le &S\int p_{T}(x)\cdot \Vert p(\cdot\vert x)-q(\cdot\vert x) \Vert_{L_1} dx + \mathcal{O}(T^{-\alpha})
	\cr &=& S\Vert p_{\Xhat_T,\Xhat_{T+1}} - p_{X_T,X_{T+1}}\Vert_{L_1} +  \mathcal{O}(T^{-\alpha})
	\ees
	Therefore,
	\bes
	&&\mathbb{E}_{(X_t,\eta_t)_{t=0}^T}\left\| p_{\Xhat_T,\cdots,\Xhat_{T+S}} - p_{X_T,\cdots,X_{T+S}}\right\|_{L_1}^2
	\cr&  \le & \mathbb{E}_{(X_t,\eta_t)_{t=0}^T}\big[2S^2\Vert p_{\Xhat_T,\Xhat_{T+1}} - p_{X_T,X_{T+1}}\Vert_{L_1}^2 +  \mathcal{O}(T^{-2\alpha})\big]
	\cr & = & 2S^2\mathbb{E}_{(X_t,\eta_t)_{t=0}^T}\Vert p_{\Xhat_T,\Xhat_{T+1}} - p_{X_T,X_{T+1}}\Vert_{L_1}^2 + \mathcal{O}(T^{-2\alpha})
	\ees
\end{proof}

\subsection{Proof of Proposition \ref{prop7}}
\begin{proof}
	Since $G\in\mathcal{G},H\in\mathcal{H}$ are bounded functions, there exist a constant $M\ge 0$ such that:
	\bes
	\vert H(X_{t},G(\eta_{t},X_{t},s),s) - f^*(H(X_{t},X_{t+s},s)) \vert \le M,\ \ t\ge 0, \ \ 1\le s \le S
	\ees
	By Lemma \ref{lm7.5}, for $1\le s \le S$ and $t_1,t_2 >0$:
	\bes
	\Vert p_{X_{t_1}, X_{t_1+s}} - p_{X_{t_2}, X_{t_2+s}}\Vert_{L_1} \le \mathcal{O}(\mathop{\min} \left\{t_1,t_2\right\}^{-\alpha})
	\ees
	
	Then, for $t_1,t_2 > 0$:
	\begin{align*}
		&\vert \mathbb{E} [H(X_{t_1},G(\eta_{t_1},X_{t_1},s),s) - f^*(H(X_{t_1},X_{t_1+s},s))] \\
		&\ \  -\mathbb{E} [H(X_{t_2},G(\eta_{t_2},X_{t_2},s),s) - f^*(H(X_{t_2},X_{t_2+s},s))] \vert  \\
		= &\vert\mathbb{E}_{\eta\sim N(0,I_m)} \int (p_{X_{t_1}, X_{t_1+s}}(x,y) - p_{X_{t_2}, X_{t_2+s}}(x,y)) [H(x,G(\eta,x,s),s) - f^*(H(x,y,s))]dxdy\vert\\
		\le & M \cdot \Vert p_{X_{t_1}, X_{t_1+s}} - p_{X_{t_2}, X_{t_2+s}} \Vert_{L_1} \\
		\le &M\mathcal{O}(\mathop{\min} \left\{t_1,t_2\right\}^{-\alpha})
	\end{align*}

	Therefore, let $T_0 \le T-S$ be an arbitrary integer, we have:
	\begin{align*}
		&\mathop{\sup}_{G\in\mathcal{G},H\in\mathcal{H}}\vert d_s(G,H) \vert\\ 
		= &\mathop{\sup}_{G\in\mathcal{G},H\in\mathcal{H}} \vert \frac1{T-s+1} \sum\limits_{t=0}^{T-s} \mathbb{E}[H(X_{t},G(\eta_{t},X_{t},s),s) - f^*(H(X_{t},X_{t+s},s))] \\
		&\ \ \ \ \ \ \ \ \ \ \ \ \ \ \ \ \ \ \ - \mathbb{E}[ H(X_T,G(\eta_T,X_T,s),s) - f^*(H(X_T,X_{T+s},s))] \vert \\
		\le &\mathop{\sup}_{G\in\mathcal{G},H\in\mathcal{H}} \vert \frac1{T-s+1} \sum\limits_{t=0}^{T_0-1}(\mathbb{E}[H(X_{t},G(\eta_{t},X_{t},s),s) -f^*(H(X_{t},X_{t+s},s))]  \\
		&\ \ \ \ \ \ \ \ \ \ \ \ \ \ \ \ \ \ \  - \mathbb{E}[ H(X_T,G(\eta_T,X_T,s),s) - f^*(H(X_T,X_{T+s},s))])\vert \\
		&+ \mathop{\sup}_{G\in\mathcal{G},H\in\mathcal{H}} \vert \frac1{T-s+1} \sum\limits_{t=T_0}^{T-s}(\mathbb{E}[H(X_{t},G(\eta_{t},X_{t},s),s) -f^*(H(X_{t},X_{t+s},s))] \\
		&\ \ \ \ \ \ \ \ \ \ \ \ \ \ \ \ \ \ \  - \mathbb{E}[ H(X_T,G(\eta_T,X_T,s),s) - f^*(H(X_T,X_{T+s},s))])\vert\\
		\le &\frac{2MT_0}{T-s+1} + \frac{M(T-s-T_0+1)}{T-s+1}\mathcal{O}(T_0^{-\alpha}) \\
		\le &\frac{2MT_0}{T-s+1} + M\mathcal{O}(T_0^{-\alpha})
	\end{align*}
	
	Let $T_0= \lceil T^{\frac{1}{\alpha+1}} \rceil$, then:
	\bes
	\mathop{\sup}_{G\in\mathcal{G},H\in\mathcal{H}}\vert d_s(G,H) \vert \le \frac{2M(T^{\frac1{\alpha+1}}+1)}{T-s+1} + M\mathcal{O}(T^{-\frac\alpha{\alpha+1}}) = \mathcal{O}(T^{-\frac\alpha{\alpha+1}})
	\ees
\end{proof}

\subsection{Proof of Proposition \ref{prop3-1}}
\begin{proof}
	Proposition \ref{prop3-1} directly follows from Theorem 8 in \cite{mcdonald2017rademacher}.
\end{proof}

\subsection{Proof of Theorem \ref{thm2}}
\begin{proof}
	By Lemma \ref{lm8}:
	\begin{align*}
		&\mathbb{E}_{(X_t,\eta_t)_{t=0}^T}\frac1{S}\sum\limits_{s=1}^{S}\Vert p_{\Xtil_T,\Xtil_{T+s}} - p_{X_T,X_{T+s}} \Vert_{L_1}^2 \\ \le& \mathbb{E}_{(X_t,\eta_t)_{t=0}^T}\frac1{S}\sum\limits_{s=1}^{S}\frac{2}{a}D_{f}(p_{\Xtil_T,\widehat{G}(\eta_T,\Xtil_T,s)}\Vert p_{X_T,X_{T+s}}) \\
		=&\frac{2}{a}\mathbb{E}_{(X_t,\eta_t)_{t=0}^T}\dot{\mathbb{L}}_T(\widehat{G})
	\end{align*}

	After some calculations, we can decompose $\mathbb{E}_{(X_t,\eta_t)_{t=0}^T}\dot{\mathbb{L}}_T(\widehat{G})$ as:
	\bes
	\mathbb{E}_{(X_t,\eta_t)_{t=0}^T}\dot{\mathbb{L}}_T(\widehat{G}) \le \tilde{\Delta}_3 + \tilde{\Delta}_4 + \Delta_5
	\ees
	Where
	\begin{align*}
		&\tilde{\Delta}_3 = \mathbb{E}_{(X_t,\eta_t)_{t=0}^T}\mathop{\sup}_H \dot{\mathcal{L}}_T(\widehat{G},H) - \mathop{\sup}_{H\in\mathcal{H}}\dot{\mathcal{L}}_T(\widehat{G},H) \\
		&\tilde{\Delta}_4 = \mathop{\inf}_{\bar{G}\in\mathcal{G}}\dot{\mathbb{L}}_T(\bar{G}) \\
		&\Delta_5 = 2\mathbb{E}_{(X_t,\eta_t)_{t=0}^T}\mathop{\sup}_{G\in\mathcal{G},H\in\mathcal{H}} \vert \dot{\mathcal{L}}_T(G,H) - \widetilde{\mathcal{L}}(G,H)\vert
	\end{align*}
	
	Then we only need to prove that:
	\bes
	\Delta_5 = 2\mathbb{E}_{(X_t,\eta_t)_{t=0}^T}\mathop{\sup}_{G\in\mathcal{G},H\in\mathcal{H}} \vert \dot{\mathcal{L}}_T(G,H) - \tilde{\mathcal{L}}(G,H)\vert \le  \tilde{\Delta}_1 + \widetilde{\Delta}_2 
	\ees
	where
	\bes
	&& \tilde{\Delta}_1 = \mathcal{O}(T^{-\frac\alpha{\alpha+1}})
	\cr &&\tilde{\Delta}_2 = \mathcal{O}\left(\sqrt{\frac{\text{Pdim}_{\mathcal{G}}\log (T\text{B}_{\mathcal{G}})}{T}}+\sqrt{\frac{\text{Pdim}_{\mathcal{H}}\log (T\text{B}_{\mathcal{H}})}{T}}\right) 
	\ees
	%

	%
		%
	Let:
	\begin{align*}
		&\dot{\mathcal{L}}_{T,s}(G,H) = \mathbb{E}_{X_T,\eta_T} H(X_T,G(\eta_T,X_T,s),s) - \mathbb{E}_{X_T,X_{T+h}} f^*(H(X_T,X_{T+s},s))\\
		&\widetilde{\mathcal{L}}_{s}(G,H) = \frac{1} {|\Omega|}\sum\limits_{t=0}^{T-s} [H(X_{t},G(\eta_{t},X_{t}, s), s) - f^*(H(X_{t},X_{t+s},s))]\\
		&b_{t}^s(G,H) = H(X_{t},G(\eta_{t},X_{t},s),s) - f^*(H(X_{t},X_{t+s},s)) \\
		&\ \ \ \ \ \ \ \ \ \ \ \ \ \ - \mathbb{E}[H(X_{t},G(\eta_{t},X_{t},s),s) - f^*(H(X_{t},X_{t+s},s))] \\
		&b_s(G,H) = \frac{\vert\Omega\vert}{T-s+1}[\widetilde{\mathcal{L}}_{s}(G,H) - \mathbb{E}\widetilde{\mathcal{L}}_{s}(G,H)] \\
		&d_s(G,H) = \dot{\mathcal{L}}_{T,s}(G,H) - \frac{\vert\Omega\vert}{T-s+1}\mathbb{E}\widetilde{\mathcal{L}}_{s}(G,H) 
	\end{align*}
	then $b_s(G,H) = \frac1{T-s+1} \sum\limits_{t=0}^{T-s}b_{t}^s(G,H)$, and:
	\bel{delta decompose}
	\Delta_5 \le \frac2{S}\sum\limits_{s=1}^{S}\mathbb{E}\mathop{\sup}_{G\in\mathcal{G},H\in\mathcal{H}}\vert b_s(G,H)\vert + \frac2{S}\sum\limits_{s=1}^{S}\mathop{\sup}_{G\in\mathcal{G},H\in\mathcal{H}}\vert d_s(G,H)\vert 
	\cr+ 2\sum\limits_{s=1}^{s} \vert 1 - \frac{\vert\Omega\vert}{S(T-s+1)} \vert\cdot \mathbb{E}\mathop{\sup}_{G\in\mathcal{G},H\in\mathcal{H}}\vert \widetilde{\mathcal{L}}_{s}(G,H)\vert
	\eel
	
	Let ${\epsilon_t}$ be the Rademacher random variables. For $1\le s\le S$, denote the Rademacher complexity in $\mathcal{G}\times\mathcal{H}$ as:
	\bes
	\mathcal{R}_s(\mathcal{G}\times\mathcal{H}) = \mathbb{E}\mathop{\sup}_{G\in\mathcal{G},H\in\mathcal{H}} \vert \frac2{T-s+1} \sum\limits_{t=0}^{T-s} \epsilon_tb_{t}^s(G,H)\vert
	\ees
	
	By Proposition \ref{prop3-1}, we have:
	\bel{Rh}
	\mathbb{E}\mathop{\sup}_{G\in\mathcal{G},H\in\mathcal{H}} \vert b_s(G,H)\vert \le \mathcal{R}_s(\mathcal{G}\times\mathcal{H})
	\eel
	
	Fix $(X_t, \eta_t)_{t=0}^{T-1}$, define empirical metric $d$ in $\mathcal{G}\times\mathcal{H}$ such that:
	\bes 
	d((G,H),(\tilde{G},\tilde{H})) = \mathop{\sup}_{0\le t\le T-s}\vert b_{t}^s(G,H)-b_{t}^s(\tilde{G},\tilde{H})\vert
	\ees
	Let $\mathcal{G}_\delta \times \mathcal{H}_\delta$ be the $\delta$-net of $(\mathcal{G}\times\mathcal{H},d)$, $C(\delta,\mathcal{G}\times\mathcal{H},d)$ be the covering number of $\delta$-net, then by Lemma \ref{lm6}:
	\begin{align*}
		\mathcal{R}_s(\mathcal{G}\times\mathcal{H}) &=\mathbb{E}_{\left(X_t,\eta_t\right)_{t=0}^{T}} \mathbb{E}_{\left(\epsilon_t\right)_{t=0}^{T-s}}\mathop{\sup}_{G\in\mathcal{G},H\in\mathcal{H}} \vert \frac2{T-s+1} \sum\limits_{t=0}^{T-s} \epsilon_tb_{t}^s(G,H)\vert \\
		&\le 2\delta + \mathbb{E}_{\left(X_t,\eta_t\right)_{t=0}^{T}} \mathbb{E}_{\left(\epsilon_t\right)_{t=0}^{T-s}}\mathop{\sup}_{(G,H)\in \mathcal{G}_\delta \times \mathcal{H}_\delta} \vert \frac2{T-s+1} \sum\limits_{t=0}^{T-s} \epsilon_tb_{t}^s(G,H)\vert \\
		&\le 2\delta + \frac{2}{T-s+1}\mathbb{E}_{\left(X_t,\eta_t\right)_{t=0}^{T}} (2\log (2C(\delta,\mathcal{G}\times\mathcal{H},d)) \cdot \sup_{(G,H)\in \mathcal{G}_\delta \times \mathcal{H}_\delta} \sum\limits_{t=0}^{T-s} (b_{t}^s(G,H))^2)^{1/2} \\
		&\le 2\delta + \frac{4}{T-s+1}\mathbb{E}_{\left(X_t,\eta_t\right)_{t=0}^{T}} (\log C(\delta,\mathcal{G}\times\mathcal{H},d) \cdot \sup_{(G,H)\in \mathcal{G}_\delta \times \mathcal{H}_\delta} \sum\limits_{t=0}^{T-s} (b_{t}^s(G,H))^2)^{1/2}
	\end{align*}
	
	By assumption, $b_{t}^s(G,H)$ can be bounded by a constant $C_0$, we have:
	\begin{align*}
		\mathcal{R}_s(\mathcal{G}\times\mathcal{H}) &\le 2\delta + \frac4{T-s+1} \mathbb{E}_{\left(X_t,\eta_t\right)_{t=0}^{T}} ((T-s+1)C_0^2\log C(\delta,\mathcal{G}\times\mathcal{H},d))^{1/2} \\
		&\le 2\delta + \frac{4C_0}{\sqrt{T-s+1}} \mathbb{E}_{\left(X_t,\eta_t\right)_{t=0}^{T}} (\log C(\delta,\mathcal{G},d) + \log C(\delta,\mathcal{H},d))^{1/2}
	\end{align*}
	
	By Theorem 12.2 in \cite{anthony_bartlett_1999}, the covering number can be bounded as:
	\begin{align*}
		&\log C(\delta,\mathcal{G},d) \le \text{Pdim}_\mathcal{G} \log \frac{2eT\text{B}_\mathcal{G}}{\delta \text{Pdim}_\mathcal{G}}  \\
		&\log C(\delta,\mathcal{H},d) \le \text{Pdim}_\mathcal{H} \log \frac{2eT\text{B}_{\mathcal{H}}}{\delta \text{Pdim}_\mathcal{H}}
	\end{align*} 
	
	The $\text{Pdim}_\mathcal{G}$ shown above is the Pseudo dimension of $G$. 
	
	Let $\delta = T^{-1}$, therefore,
	\bes
	&& \ \ \ \mathcal{R}_s(\mathcal{G}\times\mathcal{H}) 
	\cr&& \le 2T^{-1} + \frac{4C_0}{\sqrt{T-s+1}}(\text{Pdim}_\mathcal{G} \log \frac{2eT^{2}\text{B}_{\mathcal{G}}}{ \text{Pdim}_\mathcal{G}}+\text{Pdim}_\mathcal{H} \log \frac{2eT^{2}\text{B}_{\mathcal{H}}}{ \text{Pdim}_\mathcal{H}})^{1/2}
	\cr&& \le 2T^{-1} + \frac{4C_0}{\sqrt{T-s+1}}(\text{Pdim}_\mathcal{G} \log \frac{2eT^{2}\text{B}_{\mathcal{G}}}{ \text{Pdim}_\mathcal{G}})^{1/2} + \frac{4C_0}{\sqrt{T-s+1}}(\text{Pdim}_\mathcal{H} \log \frac{2eT^{2}\text{B}_{\mathcal{H}}}{ \text{Pdim}_\mathcal{H}})^{1/2}
	\cr&& = \mathcal{O}((\frac{\text{Pdim}_\mathcal{G}\log (T\text{B}_{\mathcal{G}})}{T})^{1/2}) + \mathcal{O}((\frac{\text{Pdim}_\mathcal{H}\log (T\text{B}_{\mathcal{H}})}{T})^{1/2})
	\ees
	
	

	
	By Proposition \ref{prop7}, we have:
	\bes
	\mathop{\sup}_{G\in\mathcal{G},H\in\mathcal{H}}\vert d_s(G,H) \vert \le \ \mathcal{O}(T^{-\frac\alpha{\alpha+1}})
	\ees

	Finally, note that $\vert\Omega\vert = \frac{S(2T-S+1)}{2}$, then:
	\bes
	\vert 1 - \frac{\vert\Omega\vert}{S(T-s+1)} \vert \mathbb{E}\mathop{\sup}_{G\in\mathcal{G},H\in\mathcal{H}}\vert \widetilde{\mathcal{L}}_{s}(G,H)\vert &\le& \vert 1 - \frac{\vert\Omega\vert}{S(T-s+1)} \vert \cdot \frac{(T-s+1)M}{\vert\Omega\vert}
	\cr & = &\frac{\vert S-2s+1\vert}{2(T-s+1)} \cdot \frac{2(T-s+1)M}{S(2T-S+1)} 
	\cr & = &\mathcal{O}(T^{-1})
	\ees
	
	therefore,
	\bes
	\Delta_5 &\le& \mathcal{O}\left(\sqrt{\frac{\text{Pdim}_{\mathcal{G}}\log (T\text{B}_{\mathcal{G}})}{T}}+\sqrt{\frac{\text{Pdim}_{\mathcal{H}}\log (T\text{B}_{\mathcal{H}})}{T}}\right) + \mathcal{O}(T^{-\frac\alpha{\alpha+1}}) + \mathcal{O}(T^{-1})
	\cr &= & \mathcal{O}\left(\sqrt{\frac{\text{Pdim}_{\mathcal{G}}\log (T\text{B}_{\mathcal{G}})}{T}}+\sqrt{\frac{\text{Pdim}_{\mathcal{H}}\log (T\text{B}_{\mathcal{H}})}{T}}\right) + \mathcal{O}(T^{-\frac\alpha{\alpha+1}})
	\cr &=& \tilde{\Delta}_1 + \tilde{\Delta}_2 
	\ees
\end{proof}

\subsection{Proof of Theorem \ref{mainthm}}
\begin{proof}
	By Proposition \ref{prop6} and Theorem \ref{thm2}, we have:
	\bes
	&&\mathbb{E}_{(X_t,\eta_t)_{t=0}^T}\left\| p_{\Xhat_T,\cdots,\Xhat_{T+S}} - p_{X_T,\cdots,X_{T+S}}\right\|_{L_1}^2 \cr&\le& 2S^2\mathbb{E}_{(X_t,\eta_t)_{t=0}^T}\Vert p_{\Xhat_T,\Xhat_{T+1}} - p_{X_T,X_{T+1}}\Vert_{L_1}^2 + \mathcal{O}(T^{-2\alpha})
	\cr &\le& \tilde{\Delta}_1 + \Delta_2 + \Delta_3 + \Delta_4 + \mathcal{O}(T^{-2\alpha})
	\cr &=& \Delta_1 + \Delta_2 + \Delta_3 + \Delta_4
	\ees
	where,
	\bes
	&&\Delta_1 =  \mathcal{O}( T^{-\frac{\alpha}{\alpha+1}}+T^{-2\alpha})
	\cr && \tilde{\Delta}_1 = \mathcal{O}(T^{-\frac{\alpha}{\alpha+1}})
	\cr&&\Delta_2 = \mathcal{O}\left(\sqrt{\frac{\text{Pdim}_{\mathcal{G}_1}\log (T\text{B}_{\mathcal{G}_1})}{T}}+\sqrt{\frac{\text{Pdim}_{\mathcal{H}_1}\log (T\text{B}_{\mathcal{H}_1})}{T}}\right)
	\cr&&\Delta_3 =\frac{4S^2}{a}\mathbb{E}_{(X_t, \eta_t)_{t=0}^T} (\mathop{\sup}_{h} \mathcal{L}_T(\ghat,h) - \mathop{\sup}_{h\in \mathcal{H}_1} \mathcal{L}_T(\ghat,h))
	\cr&&\Delta_4 = \frac{4S^2}{a}\mathop{\inf}_{\bar{g}\in \mathcal{G}_1}\mathbb{L}_T(\bar{g})
	\ees

\end{proof}

\subsection{Proof of Theorem \ref{thm14}}
\begin{proof}
	Similar to the proof of Theorem \ref{thm2}, we can get that:
	\bes
	\mathbb{E}\frac1{T}\sum\limits_{t=0}^{T-1} \left\| p_{X_{i,t}, \widehat{g}(\eta_t, X_{i,t})} - p_{X_{i,t},X_{i,t+1}}\right\|_{L_1}^2 \le \ddot{\Delta}_1 + \ddot{\Delta}_2 + \ddot{\Delta}_3
	\ees
	where,
	\bes
	&& \ddot{\Delta}_1 = \mathcal{O}\left(\sqrt{\frac{\text{Pdim}_{\mathcal{G}}\log (n\text{B}_{\mathcal{G}})}{n}}+\sqrt{\frac{\text{Pdim}_{\mathcal{H}}\log (n\text{B}_{\mathcal{H}})}{n}}\right)
	\cr &&\ddot{\Delta}_2 = \mathcal{O}(1)\cdot\mathbb{E}(\mathop{\sup}_{h} \dot{\mathcal{L}}_{(T)}(\widehat{g},h)- \mathop{\sup}_{h\in\mathcal{H}}\dot{\mathcal{L}}_{(T)}(\widehat{g},h))
	\cr &&\ddot{\Delta}_3 = \mathcal{O}(1)\cdot \mathop{\inf}_{\bar{g}\in\mathcal{G}} \dot{\mathbb{L}}_{(T)}(\bar{g})
	\ees
	
	Let,
	\bes
	&&\ghat(\eta_0, x)\sim p(\cdot\vert x)
	\cr &&X_1\vert (X_0=x) \sim q(\cdot\vert x)
	\ees
	Let $M_0 = \mathop{\sup}_{0\le s\le S-1} \int \frac{p_{T+s}^2(x)}{p_{T-1}(x)}dx < \infty$, by the proof of Proposition \ref{prop6}, we have:
	\bes
	&& \Vert p_{\Xhat_T,\cdots,\Xhat_{T+S}} - p_{X_T,\cdots, X_{T+S}}\Vert_{L_1}^2
	\cr&=& (\int \vert p_{\Xhat_T,\cdots,\Xhat_{T+S}}(x_0,\cdots,x_S) - p_{X_T,\cdots, X_{T+S}}(x_0,\cdots,x_S)\vert dx_{0:S})^2
	\cr&\le&(\int \sum\limits_{r=0}^{S-1} p_{T+r}(x)\cdot \Vert p(\cdot\vert x)-q(\cdot\vert x)\Vert_{L_1}dx)^2
	\cr&\le& \int \frac{(\sum\limits_{r=0}^{S-1} p_{T+r}(x))^2}{p_{T-1}(x)}dx \cdot \int p_{T-1}(x)\Vert p(\cdot\vert x)-q(\cdot\vert x)\Vert_{L_1}^2dx
	\cr &\le& S^2M_0\int p_{T-1}(x)\cdot D_f(p(\cdot\vert x)\Vert q(\cdot\vert x))dx
	\cr&=& S^2M_0D_f(p_{X_{T-1},\ghat(\eta_{T-1},X_{T-1})}\Vert p_{X_{T-1},X_T})
	\cr&\le & S^2M_0T(\ddot{\Delta}_1+\ddot{\Delta}_2 + \ddot{\Delta}_3)
	\cr&=& \ddot{\Delta}_1+\ddot{\Delta}_2 + \ddot{\Delta}_3
	\ees
\end{proof}

\subsection{Proof of Proposition \ref{prop: delta3}}
\begin{proof}
	We first denote $h_{\hat{g}} := \arg \sup\limits_{h} \calL(\hat{g},h)$ and  $h_{\hat{g}} $ is continuous on  $E_2 = [-\log T, \log T]^{2p+1}$ by assumption. Let $E = E_2$, $L = \log T$ and $N = T ^{\frac{2p+1}{2(2p + 3)}}/\log T$ in the  Theorem 4.3 in \cite{shen2019deep}, there exists a ReLU network  $\hat{h}_\phi \in \calH_1$  with depth $\widetilde{\calD} = 12\log T  + 14 + 2(2p+1)$ and width $ \widetilde{\calW} = 3^{2p+4} \max\{(2p+1)\lfloor (T ^{\frac{2p+1}{2(2p+3)}}/\log T)^{\frac{1}{2p+1}}\rfloor,T ^{\frac{2p+1}{2(2p+3)}}/\log T + 1\}$, such that
	\bes
	\|\hat{h}_\phi - h_{\hat{g}}\|_{L^{\infty}(E_2)} \leq 19 \sqrt{2p+1} w_{h_{\hat{g}}}^{E_2} (2\log T\cdot  T^{\frac{-1}{2p+3}})
	\ees
	where  $w_{h_{\hat{g}}}^{E_2}$ is the modulus of $h_{\hat{g}}$ as defined in \cite{shen2019deep}. Then by continuity of $\mathcal{L}$, 
	\bes
	\Delta_3 = \mathop{\sup}_{h} \mathcal{L}(\ghat,h) - \mathop{\sup}_{h\in \mathcal{H}_1} \mathcal{L}(\ghat,h) \leq  \mathcal{L}(\ghat,h_{\hat{g}}) -  \mathcal{L}(\ghat,\hat{h}_\phi ) \rightarrow 0
	\ees
	Similarly, let  $g^* := \arg\inf\limits_{g} \bbL(g)$ be continous function on  $E_1 = [-\log T, \log T]^{p+m+1}$. Setting  $E = E_1$, $L = \log T$ and $N = T ^{\frac{p+m+1}{2(3+p+m)}}/\log T$ in the  Theorem 4.3 in \cite{shen2019deep}, there exists a ReLU network $\bar{g} \in \calG_1$  with  depth $\calD = 12\log T  + 14 + 2(p+m+1)$ and width $ \calW = 3^{p+m+4} \max\{(p+m+1)\lfloor (T ^{\frac{p+m+1}{2(3+p+m)}}/\log T)^{\frac{1}{p+m+1}}\rfloor,T ^{\frac{p+m+1}{2(3+p+m)}}/\log T + 1\}$,  such that  
	\bes
	\|\bar{g} - g^*\|_{L^{\infty}(E_1)} \leq 19 \sqrt{p+m +1} w_{g^*}^{E_1} (2\log T\cdot T^{\frac{-1}{p+m+3}})
	\ees
	where $w_{g^*}^{E_1}$ is the modulus of $g^*$ as defined in \cite{shen2019deep}. Let $\bar{h} = f'(\frac{p_{X_T, \bar{g}(\eta, X_T)}}{p_{X_T, X_{T+1}}})$ and $h^* = f'(\frac{p_{X_T, g^*(\eta, X_T)}}{p_{X_T, X_{T+1}}})$, by the $f'$ is continuous  and integrable in  $L_1$  , we have $\|\bar{h} - h^*\| \rightarrow 0$ as $T \rightarrow \infty$. 
	By the definition,  $\bbL(\cdot)$  can be rephrase as  following
	\bes
	\mathop{\inf}_{\bar{g}\in \mathcal{G}_1} \mathbb{L}(\bar{g}) = \mathbb{E}_{X_T,\eta_T} \bar{h}(X_T,\bar{g}(\eta_T,X_T))  &- \mathbb{E}_{X_T,X_{T+1}} f^*(\bar{h}(X_T,X_{T+1})).
	\ees
	Therefore, by the continuity of $f^*$ (since $f$ is a differentiable convex function), we have
	\bes
	\Delta_4 = \mathop{\inf}_{\bar{g}\in \mathcal{G}_1} \mathbb{L}(\bar{g}) 
	\rightarrow 0.
	\ees
	
\end{proof}

\subsection{Additional lemmas}
%
%

\begin{lem}\label{lm6}
	Let $\epsilon_i(1\le i\le m)$ be the Rademacher random variables. For any $A\in \mathbb{R}^m$, let $R = \sup_{a\in A} (\sum\limits_{i=1}^m a_i^2)^{1/2}$. Then:
	\bes
	\mathcal{R}(A) = \mathbb{E}_{(\epsilon_i)_{i=1}^m}\mathop{\sup}_{a\in A} \vert \frac1m \sum\limits_{i=1}^m\epsilon_i a_i \vert \le \frac{R\sqrt{2\log (2\vert A\vert)}}{m}
	\ees 
\end{lem}
\begin{proof}
	Let $B = A \cup (-A)$, we only need to prove that:
	\bes
	\mathbb{E}_{(\epsilon_i)_{i=1}^m} [\mathop{\sup}_{b\in B} \frac1m \sum\limits_{i=1}^m\epsilon_i b_i] \le \frac{R\sqrt{2\log \vert B\vert}}{m}
	\ees
	For arbitrary $s$, by Jensen's inequality:
	\begin{align*}
		\exp (s\mathbb{E}_{(\epsilon_i)_{i=1}^m} [\mathop{\sup}_{b\in B}  \sum\limits_{i=1}^m\epsilon_i b_i]) &\le \mathbb{E}_{(\epsilon_i)_{i=1}^m} [\exp\left\{s\mathop{\sup}_{b\in B}  \sum\limits_{i=1}^m\epsilon_i b_i\right\}] \\
		&\le \sum\limits_{b\in B}\mathbb{E}_{(\epsilon_i)_{i=1}^m}[\exp\left\{s\sum\limits_{i=1}^m\epsilon_i b_i\right\}] \\
		&= \sum\limits_{b\in B} \prod\limits_{i=1}^m \mathbb{E}_{(\epsilon_i)_{i=1}^m} \exp\left\{s\epsilon_i b_i\right\}
	\end{align*}
	
	Because $E[\epsilon_i b_i] = 0$, and $\epsilon_i b_i \in [-\vert b_i \vert, \vert b_i \vert]$, then applying Hoeffding's inequality, we have:
	\begin{align*}
		\exp \left\{s\mathbb{E}_{(\epsilon_i)_{i=1}^m} [\mathop{\sup}_{b\in B}  \sum\limits_{i=1}^m\epsilon_i b_i]\right\} &\le \sum\limits_{b\in B} \prod\limits_{i=1}^m \mathbb{E}_{(\epsilon_i)_{i=1}^m} \exp\left\{s\epsilon_i b_i\right\} \\
		&\le  \sum\limits_{b\in B} \prod\limits_{i=1}^m \exp\left\{\frac{s^2(2\vert b_i\vert^2)}{8}\right\} \\
		&= \sum\limits_{b\in B} \exp\left\{\frac{s^2}{2}\sum\limits_{i=1}^m b_i^2\right\} \\
		&\le \vert B \vert \exp\left\{\frac{s^2R^2}{2}\right\}
	\end{align*}
	Therefore,
	\bes
	\mathbb{E}_{(\epsilon_i)_{i=1}^m} [\mathop{\sup}_{b\in B}  \sum\limits_{i=1}^m\epsilon_i b_i] \le \frac{\log \vert B\vert}{s} + \frac{sR^2}{2}
	\ees
	Let $s = \frac{\sqrt{2\log \vert B\vert}}{R}$, we have:
	\bes
	\mathbb{E}_{(\epsilon_i)_{i=1}^m} [\mathop{\sup}_{b\in B}  \sum\limits_{i=1}^m\epsilon_i b_i] \le R\sqrt{2\log \vert B\vert}
	\ees
\end{proof}

%

\begin{lem}\label{lm7.5}
	Suppose Assumption \ref{a2} holds. Then for $1\le s\le S$ and $t_1, t_2 > 0$, the joint density function of $(X_{t_1},X_{t_1+s})$ and $(X_{t_2},X_{t_2+s})$ satisfy:
	\bel{pairwise pdf converges}\Vert p_{X_{t_1},X_{t_1+s}} - p_{X_{t_2},X_{t_2+s}}\Vert_{L_1} \le \mathcal{O}(\mathop{\min}\left\{t_1, t_2\right\}^{-\alpha})
	\eel
\end{lem}
\begin{proof}
	Let $p_s(\cdot\vert x)$ be the conditional density function of $X_s\vert(X_0=x)$.
	
	For $1\le s\le S$ and $t\ge 0$, the conditional density function of $X_{t+s}\vert X_t$ satisfies:
	\begin{align*}
		p_{X_{t+s}\vert X_t}(x_s\vert x_0) &= \int p_{X_{t+s},X_{t+s-1},\cdots, X_{t+1}\vert X_t}(x_s,x_{s-1},\cdots,x_1\vert x_0) dx_{s-1}\cdots dx_1\\
		&= \int \prod\limits_{i=1}^{s}p_1(x_i\vert x_{i-1}) dx_{s-1}\cdots dx_1 \\
		&= p_s(x_s\vert x_0)
	\end{align*}
	
	Then, for $1\le s\le S$ and $t_1, t_2 > 0$, we have:
	\begin{align*}
		\Vert p_{X_{t_1},X_{t_1+s}} - p_{X_{t_2},X_{t_2+s}}\Vert_{L_1} & = \int \vert p_{t_1}(x)\cdot p_s(y\vert x) - p_{t_2}(x)\cdot p_s(y\vert x)\vert dy dx \\
		&= \int \vert p_{t_1}(x)- p_{t_2}(x)\vert \cdot p_s(y\vert x) dy dx \\
		&= \int \vert p_{t_1}(x) - p_{t_2}(x)\vert dx \\
		&= \Vert p_
		{t_1}-p_{t_2}\Vert_{L_1}
	\end{align*}
	
	Combining Assumption \ref{a2}, we have:
	\bes\Vert p_{X_{t_1},X_{t_1+s}} - p_{X_{t_2},X_{t_2+s}}\Vert_{L_1} &=&\Vert p_{t_1}-p_{t_2}\Vert_{L_1} 
	\cr & \le& \Vert p_{t_1}-p_\infty\Vert_{L_1} + \Vert p_{t_2}-p_\infty \Vert_{L_1}
	\cr &\le& \mathcal{O}(t_1^{-\alpha})+\mathcal{O}(t_2^{-\alpha})
	\cr & =& \mathcal{O}(\mathop{\min}\left\{t_1, t_2\right\}^{-\alpha})
	\ees
\end{proof}

\begin{lem}\label{lm8}
	Suppose convex function $f$ satisfies $f(1)=0$. If equation (\ref{lm0-0}) holds, then for any density functions $p$ and $q$, we have:
	\bel{lm8-1}
	D_f(p\|q) \ge \frac{a}2 \left\|p-q \right\|_{L_1}^2
	\eel
\end{lem}
\begin{proof}
	Let $L = f'(0)$. By equation (\ref{lm0-0}), we can easily get the following inequality:
	\bel{lm8pf-1}
	f(x+1) \ge \frac{a}2 \frac{x^2}{1+bx} + Lx, \ \ x\ge -1
	\eel
	Let $r(x) = p(x)/q(x)-1 \ge -1$, then the $f$ divergence of $p$ from $q$ can be expressed as:
	\bes
	D_f(p\Vert q) = \int q(x) f(r(x)+1) dx
	\ees
	Note that,
	\bel{lm8pf-2}
	\int r(x)q(x)dx = \int (p(x)- q(x))dx = 0
	\eel
	combining with inequality (\ref{lm8pf-1}), we have,
	\bes
	D_f(p\Vert q) &=& \int q(x) f(r(x)+1) dx
	\cr &\ge& \int q(x) (\frac{a}2 \frac{r(x)^2}{1+br(x)} + Lr(x)) dx
	\cr &=& \frac{a}2 \int q(x)  \frac{r(x)^2}{1+br(x)} dx
	\ees
	By equation (\ref{lm8pf-2}),
	\bes
	\int q(x)(1+br(x))dx = \int q(x)dx + b\int q(x)r(x)dx = 1
	\ees
	Since $0<b<1$, $1+br(x) \ge 1-b >0$ holds for all $x$. Therefore, according to Cauchy's inequality, we have:
	\bes
	D_f(p\Vert q) &\ge& \frac{a}2 \int q(x)  \frac{r(x)^2}{1+br(x)} dx
	\cr &=& \frac{a}2 \int q(x)  \frac{r(x)^2}{1+br(x)} dx \cdot \int q(x)(1+br(x))dx 
	\cr &\ge& \frac{a}2 (\int q(x) \vert r(x) \vert dx)^2
	\cr &=& \frac{a}2 (\int \vert p(x)-q(x) \vert dx)^2
	\cr &=& \frac{a}2 \Vert p-q \Vert_{L_1}^2
	\ees
\end{proof}

\section{Implementations}\label{appendix:implement}
\subsection{Simulations}
\noindent
We present here the visualization of $\phi_1$, $\phi_e$, and $\Sigma_{\infty}$ of Case 1 in the simulation study.
\begin{figure}[H]
	\centering
	\includegraphics[width=1\columnwidth]{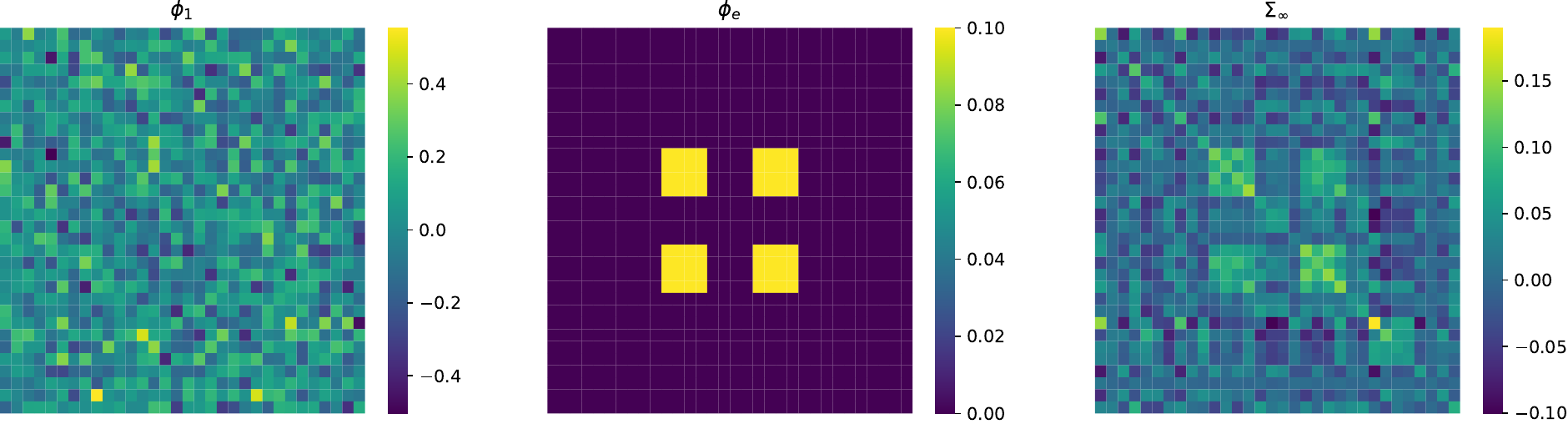}
	\caption{The visualization of $\phi_1$, $\phi_e$, and $\Sigma_{\infty}$ of Case 1 in the simulation study.}
	\label{fig:simulation-fig1}
\end{figure}
\subsection{The ADNI study}\noindent
Given two images $X$ and $Y$, structural similarity index measure (SSIM) is defined as
\bes
\text{SSIM} (X,Y) = \frac{(2\mu_x \mu_y + c_1)(2\sigma_{xy}+c_2)}{(\mu_x^2+\mu_y^2+c_1)(\sigma_x^2+\sigma_y^2+c_2)},
\ees
where $(\mu_x,\mu_y)$, $(\sigma_x^2,\sigma_y^2)$ are the mean and variance of pixel values in $X$ and $Y$ respectively.  The $\sigma_{xy}$ is the covariance of $X$ and $Y$. The $c_1 = (0.01R)^2, c_2  = (0.03R)^2$ where $R$ denotes the range of pixel values in the image. In the computing of $\text{SSIM} (\hX_{s},X_{s})$, $R$ refers to the pixel range of $X_{s}$.

Below  are the specifics of the Generator and the Discriminator used in the ADNI study.\\
\noindent\textbf{Generator: } The Generator $G \in \calG$ consists of two component: the Encoder $E_G$ and the  Decoder $D_G$. Initially, a 2D slice $X_t$ is fed into the $E_G$, which generate an embedding vector of size 130, denoted as $E_G(X_t) \in \bbR^{130}$. Next we concatenate $E_G(X_t)$ with age difference vector $\bs$ and use it as the input of $D_G$. The output of  $D_G$ is a  generated image $\hX_{t+s}$ with the same dimension of $X_t$. The structure of  Encoder $E_G$  and Decoder $D_G$ are adopted from residual U-net proposed by  \cite{zhang2018road}.\\
\textbf{Discriminator: } The Discriminator $H \in \calH$ consists of  a encoder part ($E_H$) and a critic part ($C_H$). At the encoder part, we obtain two latent features: $E_H(X_t)$ and $E_H(X_{t+s})$. Then we consider the combination of  $E_H(X_t)$, $E_H(X_{t+s})$ and $\bs$ as the input of the critic $C_H$, which produce a confident score. The encoder components $E_H(X_t)$ and $E_H(X_{t+s})$ resemble  the encoder in the  generator, while the critic part is adopted from \cite{zhang2019self}.

The number of training epochs was set to 500.  We use  AdamW optimizer  \cite{loshchilov2017decoupled} for both networks with  a learning rate and weight decay of $1\times 10^{-4}$. The size of mini-batches is set to 12. 
\end{document}